\documentclass[nohyperref]{article}

\usepackage{microtype}
\usepackage{graphicx}
\usepackage{subfigure}
\usepackage{booktabs} %

\usepackage[colorlinks=true,linkcolor=blue,citecolor=blue]{hyperref}

\usepackage[accepted]{icml2022}

\usepackage{amsmath}
\usepackage{amssymb}
\usepackage{mathtools}
\usepackage{amsthm}
\usepackage{enumitem}

\usepackage[capitalize,noabbrev]{cleveref}

\theoremstyle{plain}
\newtheorem{theorem}{Theorem}[section]

\newtheorem{lemma}[theorem]{Lemma}

\theoremstyle{definition}

\newtheorem{assumption}[theorem]{Assumption}
\theoremstyle{remark}

\usepackage[textsize=tiny]{todonotes}

\usepackage{amsthm}
\usepackage[export]{adjustbox}
\usepackage{amsmath}
\usepackage{amssymb}
\usepackage{colortbl}
\usepackage{wrapfig}
\usepackage{algorithm}
\usepackage{xcolor}
\usepackage{dsfont}
\usepackage{cancel}

\newcommand{\ak}[1]{{\textcolor{black}{#1}}}

\newcommand{\policy}{\pi}
\newcommand{\mdp}{\mathcal{M}}
\newcommand{\states}{\mathcal{S}}
\newcommand{\actions}{\mathcal{A}}

\newcommand{\behavior}{{\pi_\beta}}

\newcommand{\hatbehavior}{\hat{\pi}_\beta}

\newcommand{\bx}{\mathbf{x}}

\newcommand{\bs}{\mathbf{s}}
\newcommand{\ba}{\mathbf{a}}

\newcommand{\indicator}{\mathds{1}}

\newcommand{\methodname}{CDS}

\newcommand{\uds}{UDS}

\usepackage{amsmath,amsfonts,bm}

\def\eqref#1{equation~\ref{#1}}

\def\1{\bm{1}}

\DeclareMathAlphabet{\mathsfit}{\encodingdefault}{\sfdefault}{m}{sl}
\SetMathAlphabet{\mathsfit}{bold}{\encodingdefault}{\sfdefault}{bx}{n}

\definecolor{Gray}{gray}{0.9}
\newcommand{\rebuttal}[1] {{\color{black} #1}}
\newcommand{\final}[1] {{\color{black} #1}}

\usepackage{titlesec}
\titlespacing\section{0pt}{0pt plus 2pt minus 2pt}{0pt plus 2pt minus 2pt}
\titlespacing\subsection{0pt}{3pt plus 4pt minus 2pt}{0pt plus 2pt minus 2pt}
\titlespacing\subsubsection{0pt}{3pt plus 4pt minus 2pt}{0pt plus 2pt minus 2pt}


\makeatletter
\newcommand*{\rom}[1]{\expandafter\@slowromancap\romannumeral #1@}
\makeatother
\newcommand{\CC}{\cellcolor{Gray}}

\icmltitlerunning{How to Leverage Unlabeled Data in Offline Reinforcement Learning}

\begin{document}

\twocolumn[
\icmltitle{How to Leverage Unlabeled Data in Offline Reinforcement Learning}

\icmlsetsymbol{equal}{*}

\begin{icmlauthorlist}
\icmlauthor{Tianhe Yu}{equal,yyy,comp}
\icmlauthor{Aviral Kumar}{equal,sch,comp}
\icmlauthor{Yevgen Chebotar}{comp}
\icmlauthor{Karol Hausman}{comp}
\icmlauthor{Chelsea Finn}{yyy,comp}
\icmlauthor{Sergey Levine}{sch,comp}
\end{icmlauthorlist}

\icmlaffiliation{yyy}{Stanford University}
\icmlaffiliation{sch}{UC Berkeley}
\icmlaffiliation{comp}{Google Research}

\icmlcorrespondingauthor{Tianhe Yu}{tianheyu@cs.stanford.edu}
\icmlcorrespondingauthor{Aviral Kumar}{aviralk@berkeley.edu}

\icmlkeywords{Machine Learning, ICML}

\vskip 0.3in
]

\printAffiliationsAndNotice{\icmlEqualContribution} %

\begin{abstract}
Offline reinforcement learning (RL) can learn control policies from static datasets but, like standard RL methods, it requires reward annotations for every transition. In many cases, labeling large datasets with rewards may be costly, especially if those rewards must be provided by human labelers, while collecting diverse unlabeled data might be comparatively inexpensive. How can we best leverage such unlabeled data in offline RL? One natural solution is to learn a reward function from the labeled data and use it to label the unlabeled data. In this paper, we find that, perhaps surprisingly, a much simpler method that simply applies zero rewards to unlabeled data leads to effective data sharing both in theory and in practice, without learning any reward model at all. While this approach might seem strange (and incorrect) at first, we provide extensive theoretical and empirical analysis that illustrates how it trades off reward bias, sample complexity and distributional shift, often leading to good results. We characterize conditions under which this simple strategy is effective, and further show that extending it with a simple reweighting approach can further alleviate the bias introduced by using incorrect reward labels. Our empirical evaluation confirms these findings in simulated robotic locomotion, navigation, and manipulation settings.

\end{abstract}

\vspace{-0.15cm}
\section{Introduction}
\vspace{-0.1cm}
Offline reinforcement learning (RL) provides the promise of a fully data-driven framework for learning performant policies. To avoid costly active data collection and exploration, offline RL methods utilize a previously collected dataset to extract the best possible behavior, making it feasible to use RL to solve real-world problems where active exploration is expensive, dangerous, or otherwise infeasible~\citep{zhan2021deepthermal, de2021discovering, Wang2018SupervisedRL, kalashnikov2018scalable}. 
However, this concept is only viable when a significant amount of data for the target task is available in advance. A more realistic scenario might allow for a much smaller amount of task-specific data, combined with a large amount of task-agnostic data, that is not labeled with task rewards and some of which may not be relevant. For example, if our goal is to train a robot to perform a new manipulation task (e.g., cutting an onion), we might have some data of the robot (suboptimally) attempting that task, perhaps collected under human teleoperation and manually labeled with rewards, combined with background data, some of which might be structurally related (e.g., picking up an onion, or cutting a carrot). {This scenario presents several questions: How do we decide which prior data should be included when learning the new task? And how do we determine reward labels to use for this prior data?}

Prior methods have attempted to answer these two questions, typically in isolation. For the first question, it has been noted that a na\"{i}ve strategy of sharing all of the prior data can be highly suboptimal~\citep{kalashnikov2021mt}, even when it is annotated with reward labels, and some works have proposed both manual~\citep{kalashnikov2021mt} and automated~\citep{yu2021conservative,eysenbach2020rewriting} data sharing strategies that prioritize the most structurally similar prior data. 
However, these methods assume that shared data can automatically be relabeled with the reward function for the task,
but the assumption that we have access to the functional form of this reward is a strong one: for example, in many real-world settings, computing the reward might require human labeling~\citep{cabi2019scaling,finn2016deep}. 
Prior works use learned classifiers for reward labeling~\citep{VICEFu2018,xie2018few,singh2019end}, or other automated mechanisms~\citep{konyushkova2020semi}. But these mechanisms themselves add complexity and potential brittleness to the pipeline. Thus, we aim to study the possibility of tackling this problem with a simple method that determines which rewards to use for shared data, with minimal supervision and no additional modeling and learning.

\begin{figure}[ht]
    \centering
    \includegraphics[width=0.48\textwidth]{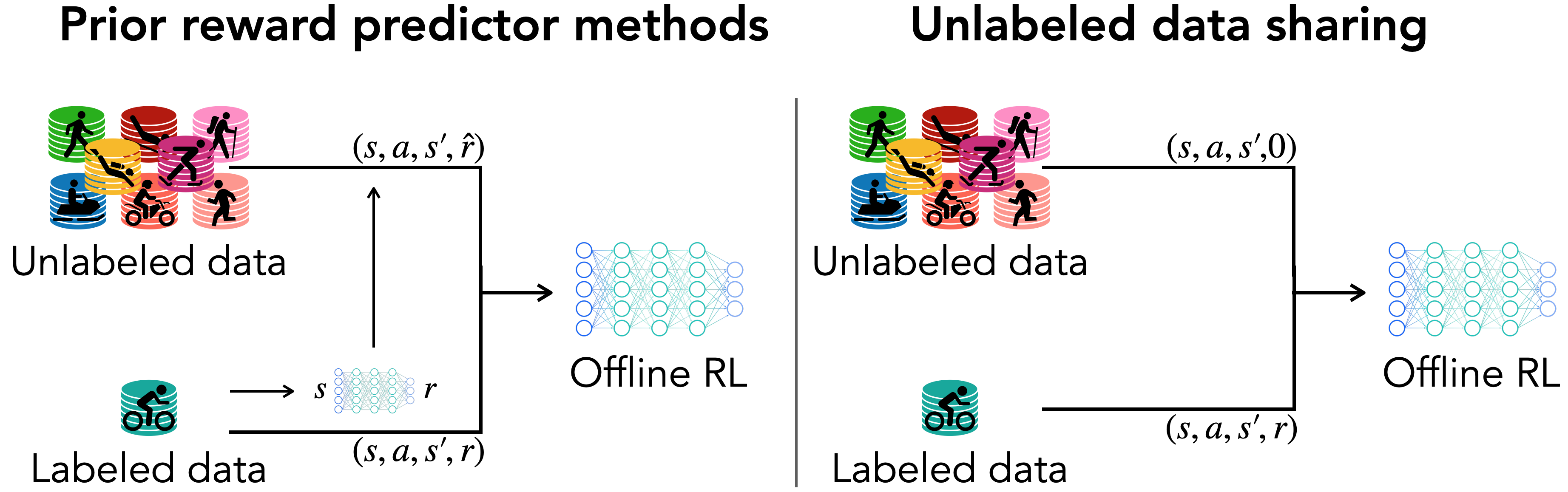}
    \vspace{-0.6cm}
    \caption{\footnotesize  Overview of various methods dealing with unlabeled offline data. UDS is a simple method compared to the complex reward learning approaches yet achieves effective results.}
    \vspace{-0.55cm}
    \label{fig:teaser}
\end{figure}

In this paper, we make the potentially surprising observation that prior data such as data from other tasks can be utilized with na\"{i}ve constant reward labels in offline RL, which corresponds to zero reward or whatever is the minimum reward for the task (which we can set to zero via rescaling, without loss of generality). It may at first seem that such an approach would lead to incorrect solutions due to using the wrong reward values. However, we show that this simple reward relabeling strategy, which we refer to as unlabeled data sharing (UDS), can outperform more sophisticated methods that separately learn a reward model and lead to good results in a range of settings. We visualize UDS along with more complex methods in Figure~\ref{fig:teaser}. Na\"{i}ve UDS does not work in all cases, but we further show that carefully reweighting the unlabeled transitions (to modify their distribution) can significantly increase the applicability of this approach and reduce the effects of reward bias, while still preserving many of the benefits. {We analyze this theoretically and show that, in practice, a method based on conservative data sharing~\citep{yu2021conservative}, which reweights unlabeled data to minimize distributional shift (originally designed for use with ground truth reward labels) can be particularly effective in this role.}

Our main contribution is the finding that simply labeling unlabeled data with a reward of zero for use in offline RL is surprisingly effective in many cases. We perform extensive theoretical and empirical analysis to study conditions under which this simple approach, UDS, would either excel or fail, and we analyze how reweighting the relabeled data (while still using zero reward as the label) can increase the range of settings when UDS is successful. Our empirical evaluation, conducted over various single-task and multi-task offline RL scenarios such as locomotion, robotic manipulation from visual inputs, and ant-maze navigation shows that this simple zero reward strategy leads to improved performance, even compared to methods that use more sophisticated learning and label propagation strategies to infer rewards. This further supports our hypothesis that relabeling unlabeled data with zero reward is an effective approach for utilizing unlabeled prior data in offline RL.

\section{Related Work}

\textbf{Offline RL.} Offline RL~\citep{ernst2005tree,riedmiller2005neural,LangeGR12,levine2020offline} considers the problem of learning a policy from a static dataset without interacting with the environment, which has shown promises in many practical applications such as robotic control~\citep{kalashnikov2018scalable,Rafailov2020LOMPO}, NLP~\citep{jaques2019way}, 
and healthcare~\citep{shortreed2011informing, Wang2018SupervisedRL}. 
The main challenge of offline RL is the distributional shift between the learned policy and the behavior policy~\citep{fujimoto2018off}, which can cause erroneous value backups. To address this issue, prior methods have constrained the learned policy to be close to the behavior policy via policy regularization~\citep{liu2020provably,wu2019behavior,kumar2019stabilizing, zhou2020plas,ghasemipour2021emaq,fujimoto2021minimalist}, conservative value functions~\citep{kumar2020conservative}, 
and model-based training with conservative penalties~\citep{yu2020mopo,kidambi2020morel,swazinna2020overcoming,lee2021representation,yu2021combo}. Unlike these prior works, we study how unlabeled data can be incorporated into the offline RL framework.

\textbf{RL with unlabeled data.} Prior works tackle the problem of learning from data without reward labels via either directly imitating expert trajectories~\citep{pomerleau1988alvinn,RossB12,GAIL2016Ho}, learning reward functions from expert data using inverse RL~\citep{abbeel2004apprenticeship,ng2000irl,ziebart2008maximum,finn2016guided,AIRLFu2018,konyushkova2020semi}, or learning a reward / value classifier that discriminates successes and failures~\citep{xie2018few,VICEFu2018,singh2019end,eysenbach2021replacing} in standard online RL settings. 
{Unlike these prior works, the goal of our paper is not to propose a sophisticated new algorithm for reward learning,} but rather to study when the simple solution of using zero as the reward label can work well from offline data. \final{While} \citet{singh2020cog} also considers relabeling prior data with zero reward in the standard, single-task offline RL setting, \final{the key distinction is that the unlabeled data in \citet{singh2020cog} cannot solve the target task and \emph{actually} gets zero rewards and is thus labeled with the true reward. Unlike \citet{singh2020cog}, we show that UDS can surprisingly also be effective even when the unlabeled data is incorrectly labeled with zero reward, and discuss how it can be improved (Section~\ref{sec:rebalancing}).}
Additionally, we consider both single-task and multi-task offline RL settings.

\textbf{Data sharing.} 
Sharing data across different tasks has been found to be effective in multi-task~\citep{eysenbach2020rewriting,kalashnikov2021mt,yu2021conservative} and meta-RL~\citep{dorfman2021offline,mitchell2021offline} and it improves performance significantly in multi-task offline RL. 
Prior works share data based on learned Q-values~\citep{eysenbach2020rewriting,li2020generalized,yu2021conservative}, domain knowledge~\citep{kalashnikov2021mt} and distance to goals in goal-conditioned settings~\citep{andrychowicz2017hindsight,liu2019competitive,sun2019policy,lin2019reinforcement,chebotar2021actionable}, and the learned distance with robust inference in the offline meta-RL setting~\citep{li2019multi}. However, all of these either require access to the functional form of the reward for relabling or are limited to goal-conditioned settings. In contrast, we attempt to tackle a more general problem where reward labels are not provided for a subset of the data, either in a  multi-task or a single-task setting and find that simple approaches for relabeling data from other tasks with the constant zero can work well.

\section{Preliminaries}
\vspace{-0.15cm}
\label{sec:prelim}

\textbf{Offline RL.} Standard RL considers a Markov decision process (MDP), $\mdp =(\states, \actions, P, \gamma, R)$, where $\states$ and $\actions$ denote the state and action spaces respectively, $P(\bs' | \bs, \ba)$ denotes the dynamics, $\gamma \in [0, 1)$ is the discount factor, and $R$ correspond to the reward function. Offline RL tackles the problem of learning a policy $\pi(\ba|\bs)$ from a static dataset with $\mathcal{D}$, generated by a behavior policy $\pi_\beta(\ba|\bs)$.

\textbf{Data sharing in offline RL.} Data sharing has been considered in the multi-task offline RL setting where there is a static multi-task dataset with $\mathcal{D} = \cup_{i=1}^N \mathcal{D}_i$ where $N$ is the number of tasks. Prior works~\citep{kalashnikov2021mt,eysenbach2020rewriting,yu2021conservative} show that sharing data from different tasks to task $i$ to be conducive. To do so, these prior methods assume access to the functional form of the reward $r_i$. This is a strong assumption in practice, as it necessitates access to a functional (programmatic) form for the reward function. In offline RL, it might be desirable to simply label the reward function by hand, but then the algorithm does not have access to the functional form of the reward, and all unlabeled data also needs to be labeled by hand for use with such methods. Our aim in this paper is to utilize unlabeled data without any reward labels at all.
If however functional access to the reward \emph{is} available, a simple strategy is to na\"ively share data across all tasks, which we refer to as Sharing All. Formally, Sharing All defines the dataset of transitions relabeled from task $j$ to task $i$ as $\mathcal{D}_{j \rightarrow i}$ and the method can be then defined as
    $\mathcal{D}^\mathrm{eff}_i := \mathcal{D}_i \cup ( \cup_{j \neq i} \mathcal{D}_{j \rightarrow i})$,
where $\mathcal{D}^\mathrm{eff}_i$ denotes the effective dataset for task $i$. Therefore, the policy optimization objective in Sharing All can be written as follows:
\begin{equation*}
     \forall i \in [N], ~~\pi^*(\ba|\bs, i) := \arg \max_{\pi}~~ J_{\mathcal{D}^\mathrm{eff}_i}(\pi) - \alpha D(\pi, \pi^\mathrm{eff}_\beta),
\end{equation*}
where $\pi_\beta^\mathrm{eff}(\ba|\bs, i)$ is the effective behavior policy for task $i$ denoted as $\pi_\beta^\mathrm{eff}(\ba|\bs, i) := |\mathcal{D}^\mathrm{eff}_i(\bs, \ba)| / |\mathcal{D}^\mathrm{eff}_i(\bs)|$, $J_{\mathcal{D}^\mathrm{eff}_i}(\pi)$ denotes the average return of policy $\pi$ in the empirical MDP induced by the effective dataset, and $D(\pi, \pi^\mathrm{eff}_\beta)$ denotes a divergence measure (e.g., KL-divergence~\citep{jaques2019way,wu2019behavior}, fisher divergence~\citep{kostrikov2021offline}, MMD distance~\citep{kumar2019stabilizing} or $D_{\text{CQL}}$ from conservative Q-values~\citep{kumar2020conservative}) between the learned policy $\pi$ and the effective behavior policy $\pi_\beta^\mathrm{eff}$. Note that conservative Q-values refer to the Q-value for a given policy corresponding to a modified reward function $r(\bs, \ba) - \alpha \pi(\ba|\bs) \cdot (\pi(\ba|\bs) / \pi_\beta(\ba|\bs) - 1)$, computed on the empirical MDP. We also note that Sharing All can be easily adapted to the single-task setting where there is only one target task with labeled data $\mathcal{D}_\text{L}$ and unlabeled prior data $\mathcal{D}_\text{U}$. While data sharing tends to show promising results, it requires the assumption of the access to the functional form of the reward function. We instead focus on the data sharing problem where we do not make such an assumption and instead, only have the reward labels for originally commanded task, which we will discuss in the following section.

\section{How To Use Unlabeled Data in Offline RL}
\vspace{-0.1cm}
\label{sec:method}
Arguably the easiest approach to utilize unlabeled data in offline RL is to simply assign the lowest possible reward to all the transitions in the unlabeled data, which we will assume to be $0$ without loss of generality, and use this data for training the underlying offline RL method. We will refer to this approach as UDS. We will show that, although this strategy might seem simplistic, in fact, it can work well both in theory and practice, when the dataset composition and the task satisfies certain criteria. Not all tasks and datasets satisfy these criteria, but we will argue that many realistic offline RL problems do satisfy them. For example, UDS can work well in a problem where the labeled data consists of high-quality, near-expert demonstrations, while the unlabeled data consists of mediocre or low-reward interactions in the environment, as is often the case in robotics~\citep{xie2019improvisation}. 
We will analyze the performance of this approach theoretically in Section~\ref{sec:uds_theory} and empirically in Section \ref{sec:empirical_analysis}.

Although the simple UDS approach can perform well in some scenarios, relabeling with zero reward of course also biases the learning process. We therefore further show that reward bias can be mitigated if the unlabeled transitions are additionally reweighted, so as to change the unlabeled data distribution (while still using a label of zero for all unlabeled data). While such reweighting reduces the sample size, it can provide an overall benefit by reducing the reward bias \ak{and distributional shift}. We will derive the optimal scheme for reweighting the transitions in the unlabeled dataset that minimizes the impact of reward bias in Section~\ref{sec:rebalancing}. Surprisingly, our analysis reveals that a near-optimal solution to reducing reward bias is to combine \uds\ with an already existing reweighting scheme for multi-task offline RL with full reward information. For example, one can choose to reweight unlabeled data with the  {CDS} scheme of \citet{yu2021conservative}, which preferentially upweights transitions based on their conservative Q-values.

\vspace{-0.1cm}
\subsection{Theoretical Analysis of UDS in Offline RL}
\vspace{-0.1cm}
\label{sec:uds_theory}
In this section, we will derive a performance bounds for \uds\ that allow us to understand several cases when this simple baseline approach still works well. We will show that the increase in the effective dataset size can often outweigh the bias incurred from using the wrong reward in several practical situations. We will also discuss conditions on the data composition and the relative distributions of labeled and unlabeled data that enable UDS to be successful.
Formally, UDS trains on an effective dataset given by:
\begin{equation}
    \mathcal{D}^\mathrm{eff} = \mathcal{D}_\mathrm{L} \cup \{(\bs_j, \ba_j, \bs'_j, 0) \in \mathcal{D}_{\mathrm{U}}\}. \label{eq:uds}
\end{equation}
We now present a policy improvement guarantee for UDS below, and then analyze the bound under special conditions. Our theoretical result builds on techniques for showing safe policy improvement bounds~\citep{laroche2019safe,kumar2020conservative,yu2021conservative}. 

\begin{table*}[t]
\vspace{-0.15cm}
\caption{\footnotesize Summary of scenarios where UDS is expected to work and where it is not expected to work. \textbf{L} denotes the characteristics of labeled data, \textbf{U} denotes characteristics of unlabeled data. Limited/Abundant refers to the relative amount of data available
High-quality/medium-quality/low-quality refers to the actual performance of the behavior policy generating the datasets. Narrow/broad refers to the relative state coverage of the datasets that we study where high coverage can lead to low distributional shift during offline RL. We provide intuitions on why UDS works or does not work under each scenario and refer to the cases shown in our analysis in Section~\ref{sec:uds_theory}.}
\vspace{-0.43cm}
\begin{center}
\scriptsize
\resizebox{\textwidth}{!}{\begin{tabular}{l|r|l}
\toprule
 \textbf{Scenarios} & \textbf{UDS} & \textbf{Intuition} \\
 \midrule
  \textbf{L}: limited + high-quality + narrow, \textbf{U}: abundant + low/medium-quality + broad & \textcolor{green}{\checkmark} & increase coverage (case \textbf{\rom{2}} and \textbf{\rom{3}})\\
  \textbf{L}: limited + medium-quality + narrow, \textbf{U}: \ak{abundant} + low/medium-quality + broad &  \textcolor{red}{$\times$} & reward bias outweighs high coverage\\
  \textbf{L}: limited + high-quality + narrow, \textbf{U}: abundant + high-quality + narrow &   \textcolor{green}{\checkmark} & identical distribution (case \textbf{\rom{1}})\\
  \textbf{L}: limited + low-quality + narrow, \textbf{U}: abundant + high-quality + narrow & \textcolor{green}{\checkmark} & increase data quality (case \textbf{\rom{2}})\\
\bottomrule
\end{tabular}}
\end{center}
\vspace{-0.7cm}
\label{tbl:summary_criteria}
\normalsize
\end{table*}

\begin{theorem}[\textbf{Policy improvement guarantee for UDS}] 
\label{prop:uds_ours}
Let $\pi^*_\text{UDS}$ denote the policy learned by UDS, and let $\pi^\mathrm{eff}_\beta(\ba|\bs)$ denote the behavior policy for the combined dataset $\mathcal{D}^\mathrm{eff}$. Then with high probability $\geq 1 - \delta$, $\pi^*_\text{UDS}$ is a safe policy improvement over $\pi_\beta^\mathrm{eff}$, i.e.,
\begin{align*}
& J(\pi^*_\text{UDS}) \geq J(\pi_\beta^\mathrm{eff}) - \zeta_\text{err} +  \underbrace{\frac{\alpha}{1 - \gamma} D(\pi^*_\text{UDS}, \pi^\mathrm{eff}_\beta)}_{\text{(c): policy improvement}},\\
 & \zeta_\text{err} = \underbrace{\mathrm{RewardBias}(\pi^*_\text{UDS}, \pi^\mathrm{eff}_\beta)}_{ \text{(a)}} \\
 &+\underbrace{\mathcal{O}\left(\frac{\gamma}{(1 - \gamma)^2}\right) \mathbb{E}_{\bs, \ba \sim \widehat{d}^{\pi}}\left[\sqrt{\frac{D_{\text{CQL}}(\pi^*_\text{UDS}, \pi^\mathrm{eff}_\beta)(\bs)}{|\mathcal{D}^\mathrm{eff}(\bs)|}} \right]}_{\text{(b): sampling error}},\\
\!\!\!\!\!\!\!\!\!\!\!\!\!\!\!\!\!\!\!\!\!\!\!(a) &:= \frac{\sum_{\bs, \ba}\Delta\left(\widehat{d}^{\behavior^\mathrm{eff}}, \widehat{d}^{\pi^*_\text{UDS}}\right)  \cdot (1 - f(\bs, \ba)) \cdot r(\bs, \ba)}{1 - \gamma},
\end{align*}
where we use the notation $f(\bs, \ba) := \frac{|\mathcal{D}_\mathrm{L}(\bs, \ba)|}{|\mathcal{D}^\mathrm{eff} (\bs, \ba)|}$ and $\Delta(\widehat{d}^{\behavior^\mathrm{eff}}, \widehat{d}^{\pi^*_\text{UDS}}) = \widehat{d}^{\behavior^\mathrm{eff}}(\bs, \ba) - \widehat{d}^{\pi^*_\text{UDS}}(\bs, \ba)$. 
\end{theorem}
A proof for Theorem~\ref{prop:uds_ours} is provided in Appendix~\ref{proof:uds_proof}. Intuitively, term (a) corresponds to the potential suboptimality incurred as a result of using the wrong reward, term (b) corresponds to the suboptimality induced due to sampling error. Finally, term (c) corresponds to the policy improvement in the empirical MDP induced by the transitions in $\mathcal{D}^\mathrm{eff}$ that occurs as a result of offline RL. Intuitively, we expect that including more unlabeled data will reduce the sampling error (b) and potentially increase how much we can improve the policy (c), while potentially increasing the reward bias term (a). The key question is under which conditions we would expect the increase in bias (a) to be lower than the decrease in term (b) obtained from using the unlabeled data. We examine this question below, presenting a few common cases where we expect this tradeoff to be favorable.

\textbf{(\rom{1}) Unlabeled data is distributed identically as labeled data.} The first special case we consider is when the distribution of state-action pairs in $\mathcal{D}_\mathrm{L}$ is {identical} to the distribution of state-action pairs in the unlabeled dataset $\mathcal{D}_U$. This can arise in many application domains, such as in robotics~\citep{xie2019improvisation,dasari2020robonet}, where a large amount of offline data is available, but only a limited uniformly-selected fraction of it can be annotated with rewards. In this case, the fraction $f(\bs, \ba)$ takes on identical values for all state-action pairs, and term (a) simply reverts to be a difference between the performance of the effective behavior policy and the learned policy, in the empirical MDP. Since offline RL algorithms would improve over the effective behavior policy in the empirical MDP, this term is negative and hence, no additional suboptimality is incurred in the reward bias term. Moreover, the sampling error term \final{(b)} reduces when more data is utilized. Thus, in this case, \uds\ improves performance due to an increase in the dataset size $|\mathcal{D}^\mathrm{eff}(\bs)|$, without incurring a cost due to the wrong reward.   

\textbf{(\rom{2}) Quality of the unlabeled dataset.} In practice, we are often provided with a small amount of high-quality reward annotated demonstration data, and a lot of unlabeled prior data of low or mediocre quality. In this case, assigning a low reward to the transitions in the unlabeled dataset does not negatively affect policy performance significantly because the bias due to wrong reward is low (due to low-quality of labeled data) and reduces sampling error (term \final{(b)}). On the other hand, when the unlabeled data is of high quality and large in size compared to the labeled dataset, even then the performance of the policy $J(\pi^*_{\mathrm{UDS}})$ can be improved by using this unlabeled set since \uds\ improves $J(\pi_\beta^\mathrm{eff})$. 

\textbf{(\rom{3}) Large unlabeled datasets for long-horizon tasks.} Another scenario when \uds\ relabeling will be beneficial to do compared to not using the unlabeled data at all is in long-horizon tasks (large value of $1/(1 - \gamma)$), where a lot of unlabeled data is available, while the labeled dataset, $\mathcal{D}_\mathrm{U}$ is relatively very small. Define  $H = \frac{1}{1 - \gamma}$; then in the case that $|\mathcal{D}^\mathrm{eff}(\bs)| = \Omega(H^2) |\mathcal{D}_\mathrm{L}(\bs)|$, the sampling error term (term \final{(b)}) will consist of one less factor of $1/(1 - \gamma)$ when utilizing unlabeled data, compared to when it is not. Since the sampling error grows quadratically in the horizon, whereas the reward bias only grow linearly, a reduction in this term by increasing $|\mathcal{D}^\mathrm{eff}(\bs)|$ (i.e., denominator) can improve compared to only using the labeled data, $\mathcal{D}_\mathrm{L}$. Thus, even though the rewards may be biased, the addition of large amounts of unlabeled, diverse data in long-horizon tasks can help despite the reward bias incurred. 

We empirically verify that \uds\ indeed helps in the special cases discussed above in Section~\ref{sec:empirical_analysis}. However, there are also cases where \uds\ does not help because of large suboptimality induced due to term (a). In Table \ref{tbl:summary_criteria}, we present a summary of the scenarios under which we expect \uds\ will help or hurt performance.
As an example of the latter, when the amount of unlabeled data, $\mathcal{D}_\mathrm{U}$ is not very large compared to the labeled dataset, such that a decrease in sampling error does not outweigh the suboptimality induced by reward bias, we should not expect \uds\ to attain very good performance, which we empirically show in Appendix~\ref{app:unlabeled_dataset_size_analysis}. To handle such cases, our key idea is to re-weight the unlabeled dataset before using it for offline RL training reduce the sampling error and reward bias.

\vspace{-0.05cm}
\subsection{Reducing Reward Bias by Reweighting Data}
\label{sec:rebalancing}
\vspace{-0.05cm}
As discussed above, one way to reduce the suboptimality induced due to reward bias is by preferentially reweighting transitions in the unlabeled data. We would hope that such a scheme can provide a favorable bias-sampling error tradeoff, even though it reduces sample size. In this section, we will derive an optimized reweighting scheme that attains a favorable tradeoff. \ak{Intuitively, this optimized scheme suggests that reweighting must reduce distributional shift against the state-action marginal of the policy obtained by running offline RL on only the transitions in the labeled data.}
This scheme intuitively matches the conservative data sharing method~\citet{yu2021conservative} previously proposed for (fully-labeled) multi-task offline RL. We will show that this conservative sharing approach reduces reward bias, controls distributional shift that appears in the numerator of sampling error, and increases the sample size.   

Formally, we will derive this reweighting scheme from the perspective of optimizing the effective behavior policy, $\pi^\mathrm{eff}_\beta$ induced by the dataset $\mathcal{D}$ after preferentially sharing transitions from the unlabeled data, so as to minimize reward bias and sampling error, while improving the dataset quality. For understanding purposes, we begin by deriving the choice of $\pi^\mathrm{eff}_\beta$ that reduces only the reward bias component:
\begin{theorem}[\textbf{Optimized reward bias reduction}] 
\label{prop:reward_bias_theorem}
The optimal effective behavior policy that minimizes the bias (a) in Theorem~\ref{prop:uds_ours}, $\mathrm{RewardBias}(\pi, \pi^\mathrm{eff}_\beta)$, satisfies
\begin{align*}
    \widehat{d}^{\widehat{\pi}^\mathrm{eff}_\beta}(\bs, \ba) \propto  \sqrt{{d}_\mathrm{L}(\bs, \ba) d^{\pi}(\bs, \ba)},
\end{align*}
where $d^{\pi}$ denotes the state-action marginal of a policy $\pi$, and $d_\mathrm{L}(\bs, \ba)$ denotes the density of state-action pair $(\bs, \ba)$ under the labeled dataset.
\end{theorem}
A proof of Theorem~\ref{prop:reward_bias_theorem} is provided in Appendix~\ref{sec:small_reward_bias}. The expression implies that the state-action marginal of the effective behavior policy (i.e., the rebalanced dataset distribution) must place mass on state-action tuples that are both likely to under the learned policy $d^{\pi}$ and appear frequently in the distribution induced by the labeled dataset $\widehat{d}_\mathrm{L}$. However, note that Theorem~\ref{prop:reward_bias_theorem} only accounts for the reward bias and not the other terms pertaining to sampling error and the performance of the effective behavior policy that appear in Theorem~\ref{prop:uds_ours}. To address both of these issues, in our next theoretical result, Theorem~\ref{thm:with_all_sources}, we derive the reweighting distribution that controls all the terms.
\begin{theorem}[\textbf{Optimized reweighting unlabeled data}; Informal]
\label{thm:with_all_sources}
The optimal effective behavior policy that maximizes a lower bound on $J(\pi_\beta^\mathrm{eff}) - \left[(a) + (b)\right]$ in Theorem~\ref{prop:uds_ours} satisfies $d^{\widehat{\pi}_\beta^\mathrm{eff}}(\bs, \ba) = p^*(\bs, \ba)$, where:
\begin{align*}
    p^* = \arg\min_{p \in \Delta^{|\mathcal{S}||\mathcal{A}|}}&~ \sum_{\bs, \ba} C_1 \frac{\widehat{d}^\pi(\bs, \ba)}{\sqrt{p(\bs, \ba)}} + C_2 |\widehat{d}_\mathrm{L}(\bs, \ba)| \frac{\widehat{d}^\pi(\bs, \ba)}{p(\bs, \ba)}, 
\end{align*}
where $C_1$ and $C_2$ are universal positive constants that depend on the sizes of the labeled and unlabeled datasets.
\end{theorem}
A proof of Theorem~\ref{thm:with_all_sources} along with a more formal statement listing the constants $C_1$ and $C_2$ is provided in in {Appendix~\ref{proof:all_sources}}. The first term in the optimization objective of $p^*$ appearing above arises from the sampling error term, while the second term corresponds to the reward bias term in Theorem~\ref{prop:uds_ours}. The optimal solution for $p^*$ must place high density on $(\bs, \ba)$ pairs where $\widehat{d}^\pi(\bs, \ba)$ is high, because this would reduce the objective in the optimization problem above. This corroborates the analysis of \citet{yu2021conservative}, which shows that utilizing a reweighting scheme that reduces distributional shift (i.e., makes $\pi(\ba|\bs)/\widehat{\pi}^\mathrm{eff}_\beta(\ba|\bs)$ or \ak{equivalently}, as we find, making $\widehat{d}^\pi(\bs, \ba)/\widehat{d}^{\pi_\beta^\mathrm{eff}}(\bs, \ba)$ smaller) will control sampling error, and give rise to performance benefits. In addition, as also shown in Theorem~\ref{prop:reward_bias_theorem}, the reward bias term is also controlled when low distributional shift appears. This means that rebalancing the dataset to control distributional shift between the learned policy $\widehat{d}^\pi(\bs, \ba)$ and $\widehat{d}^{\pi^\mathrm{eff}_\beta}(\bs, \ba)$ is effective in unlabeled settings as well.

\begin{table*}[t]
\vspace{-0.2cm}
\caption{\footnotesize Results on single-task environments {hopper} and {AntMaze} from the D4RL~\citep{fu2020d4rl} benchmark. The numbers are averaged over three random seeds. We only bold the best-performing method that does not have access to the true reward for relabeling. UDS outperforms No Sharing in three out of the four settings while achieving competitive performances compared to Sharing All in all the settings. CDS+UDS further improves over UDS in three out of four settings.}
\vspace{-0.24cm}
\begin{center}
\resizebox{\textwidth}{!}{\begin{tabular}{l|l|l|rrrrrr|rr}
\toprule
& & & & & & & & & \multicolumn{2}{c}{\textbf{Oracle Methods}} \\
 \textbf{Environment} & \textbf{Labeled data} & \textbf{Unlabeled data}   & \textbf{CDS+UDS}   & \textbf{UDS}         & \textbf{No Sharing} & \textbf{Reward Pred.} & \textbf{VICE} & \textbf{RCE}  & \textbf{CDS} & \textbf{Sharing All} \\
 \midrule
 D4RL hopper & expert & random & \textbf{81.5}  & 78.6 & 77.1 & 67.6 & n/a & n/a & 83.3 & 86.1\\
  & expert & medium & \textbf{78.3}  & 64.4 & 77.1 & 51.7 & n/a & n/a & 82.5 & 64.6\\
\midrule
D4RL AntMaze & expert & medium-play & 82.6  & \textbf{82.7} & 17.2 & 0.0 & 0.0 & 0.0 & 83.5 & 83.1 \\
 & expert & large-play & \textbf{47.1} & 33.1 & 0.7 & 0.0 & 0.0 & 0.0 & 46.1 & 50.2\\
\bottomrule
\end{tabular}}
\end{center}
\vspace{-0.7cm}
\label{tbl:single_task}
\normalsize
\end{table*}

While the scheme derived above can, in principle, be implemented exactly in practice, it would require computing state-marginals. Since, computing state marginals accurately in an offline setting has been an outstanding challenge,
we instead can utilize a reweighting scheme that corrects for distributional shift approximately without needing state-marginals. One of such methods is conservative data sharing (CDS)~\citep{yu2021conservative} that can be implemented without requiring additional components beyond the machinery of the offline RL method. Formally, CDS is given by:
\begin{equation*}
    \mathcal{D}^\mathrm{eff}_i = \mathcal{D}_i \cup ( \cup_{j \neq i}\{(\bs_j, \ba_j, \bs'_j, r_i) \in \mathcal{D}_{j \rightarrow i}: \Delta^\pi(\bs, \ba) \geq 0\}),
\end{equation*}
where $\bs_j, \ba_j, \bs'_j$ denote the transition from $\mathcal{D}_j$, $r_i$ denotes the reward of $\bs_j, \ba_j, \bs'_j$ relabeled for task $i$, $\pi$ denotes the task-conditioned policy $\pi(\cdot|\cdot, i)$, $\Delta^\pi(\bs_j, \ba_j)$ is the condition that shares data only if the expected conservative Q-value of the relabeled transition exceeds the top $k$-percentile of the conservative Q-values of the original data. In our experiments and analysis in Section \ref{sec:exp}, we find that utilizing CDS improves performance over simply training with \uds\ in a number of domains supporting the theoretical analysis.

\subsection{Why Can We Expect UDS to Outperform Reward Prediction Methods?}
\label{sec:reward_predictor_discussion}
While we have discussed various conditions where we would expect UDS (with or without various reweighting methods) to improve final performance over an approach that does not utilize the unlabeled data, it is also instructive to consider how it comes to a method that instead trains an approximate reward function, $\widehat{r}_\phi(\bs, \ba)$ using the labeled data, and then uses this approximate reward to annotate the unlabeled data. In our experiments, we will show, perhaps surprisingly that UDS often outperforms prior reward learning methods. In this section, we provide some intuition for why. 
First, we note the following generic expression for reward bias (term (a) in Theorem~\ref{prop:uds_ours}) for any approach:
\begin{align*}
    \mathrm{RewardBias}(\pi, \pi^\mathrm{eff}_\beta) = \frac{\sum_{\bs, \ba} \Delta\left(\widehat{d}^{\behavior^\mathrm{eff}}, \widehat{d}^{\pi}\right) \cdot \Delta r(\bs, \ba)}{1 - \gamma},
\end{align*}
where $\Delta r(\bs, \ba)$ is the error in the reward applied to the unlabeled data, such that $\Delta r(\bs, \ba) = (1 - f(\bs, \ba)) r(\bs, \ba)$ for UDS, and $\Delta r(\bs, \ba) = r(\bs, \ba) - \widehat{r}_\phi(\bs, \ba)$ for a reward prediction method. Note that while for UDS, $\forall~\bs, \ba, ~\Delta r(\bs, \ba) \geq 0$, this is not necessarily the case for a generic reward prediction method.

\textbf{UDS.} Since $\Delta r(\bs, \ba) \geq 0$ for all state-action tuples, state-action pairs that appear more frequently under the learned policy compared to the effective behavior policy, i.e., when $\Delta\left(\widehat{d}^{\behavior^\mathrm{eff}}, \widehat{d}^{\pi}\right) < 0$, contribute to reducing the suboptimality induced due to reward bias. 

\textbf{Reward prediction.} When $\Delta r(\bs, \ba) = r(\bs, \ba) - \widehat{r}_\phi(\bs, \ba)$, $\Delta r(\bs, \ba)$ may not be positive on all state-action tuples, and thus reward prediction methods fail to reduce the contribution of such state-action pairs in term (a). Since policy optimization will seek out those policies that maximize $\widehat{d}^\pi(\bs, \ba)$ on state-action pairs with high rewards, state-action pairs where $\Delta(\widehat{d}^{\pi_\beta^\mathrm{eff}}, \widehat{d}^\pi) < 0$ (i.e., state-action pairs that appear more under the learned policy) are likely to also have $\Delta r(\bs, \ba) < 0$. This ``exploitation'' inhibits the correction of reward bias and provides an explanation for why reward prediction approaches may still not perform well due to overestimation errors in the reward function.

\section{Experiments}
\label{sec:exp}
In our experiments, we aim to evaluate whether the theoretical potential for simple minimum-reward relabeling to attain good results is reflected in benchmark tasks and more complex offline RL settings. In particular, we will study: \textbf{(1)} can UDS match or exceed the performance of sophisticated reward inference methods and methods with oracle reward functions in simulated locomotion and navigation tasks? \textbf{(2)} can an optimized reweighting strategy (e.g., CDS+UDS) further improve the results achieved by UDS?, \textbf{(3)} how does UDS behave with and without combining with an optimized reweighting strategy in multi-task offline RL settings?, \textbf{(4)} under which conditions does UDS work and not work and does optimizing for reweighting help?

\begin{table*}[ht]
\vspace{-0.2cm}
\caption{\footnotesize Results for multi-task robotic manipulation (Meta-World) and navigation environments (AntMaze) with low-dimensional state inputs. Numbers are averaged across 6 seeds, $\pm$ the 95$\%$-confidence interval. We take the results of No Sharing, Sharing All and CDS, directly from \cite{yu2021conservative}. We bold the best-performing method that does not have access to the true rewards for relabeling. We include per-task performance for Meta-World domains and the overall performance averaged across tasks (highlighted in gray) for all three domains. Both CDS+UDS and UDS outperforms prior vanilla multi-task offline RL approach (No Sharing) and reward learning methods (Reward Predictor, VICE and RCE) while performing competitively compared to oracle reward relabeling methods.}
\begin{center}
\resizebox{\textwidth}{!}{\begin{tabular}{l|l|rrrrrr|rr}
\toprule
& & & & & & & & \multicolumn{2}{c}{\textbf{Oracle Methods}} \\
\textbf{Environment} & \textbf{Tasks} & \textbf{CDS+UDS} & \textbf{UDS} & \textbf{VICE} & \textbf{RCE} & \textbf{No Sharing} & \textbf{Reward Pred.} & \textbf{CDS} & \textbf{Sharing All}\\ \midrule
& door open & \textbf{61.3\%}$\pm$7.9\% & 51.9\%$\pm$25.3\% & 0.0\%$\pm$0.0\% & 0.0\%$\pm$0.0\% & 14.5\%$\pm$12.7\% & 0.0\%$\pm$0.0\% & 58.4\%$\pm$9.3\% & 34.3\%$\pm$17.9\%\\
& door close & 54.0\% $\pm$42.5\% & 12.3\%$\pm$27.6\% & 66.7\%\%$\pm$47.1\% & 0.0\%$\pm$0.0\% & 4.0\%$\pm$6.1\% & \textbf{99.3\%}$\pm$0.9\% & 65.3\%$\pm$27.7\% & 48.3\%$\pm$27.3\%\\
Meta-World& drawer open & \textbf{73.5\%}$\pm$9.6\% & 61.8\%$\pm$16.3\% & 0.0\%$\pm$0.0\% & 0.0\%$\pm$0.0\% & 16.0\%$\pm$17.5\% & 13.3\%$\pm$18.9\% & 57.9\%$\pm$16.2\% & 55.1\%$\pm$9.4\%\\
& drawer close & 99.3\%$\pm$0.7\% & \textbf{99.6\%}$\pm$0.7\% & 19.3\%$\pm$27.3\% & 2.7\%$\pm$1.7\% & 99.0\%$\pm$0.7\% & 50.3\%$\pm$35.8\% & 99.0\%$\pm$0.7\% & 98.8\%$\pm$0.7\%\\
& \CC \textbf{average} & \CC \textbf{71.2\%} $\pm$ 11.3\% & \CC 56.4\%$\pm$12.8\% & \CC 21.5\%$\pm$0.7\% & \CC 0.7\%$\pm$0.4\% & \CC 33.4\%$\pm$8.3\% & \CC 41.0\%$\pm$11.9\% & \CC 70.1\%$\pm$8.1\% & \CC 59.4\%$\pm$5.7\%\\
\midrule
AntMaze    & medium (3 tasks)  &\textbf{31.5\%}$\pm$3.0\% & 26.5\%$\pm$9.1\%  & 2.9\%$\pm$1.0\%  & 0.0\%$\pm$0.0\%  & 21.6\%$\pm$7.1\% & 3.8\%$\pm$3.8\% & 36.7\%$\pm$6.2\%  & 22.9\%$\pm$3.6\%\\
& large (7 tasks)  & \textbf{18.4\%}$\pm$6.1\% & 14.2\%$\pm$3.9\% & 2.5\%$\pm$1.1\% & 0.0\%$\pm$0.0\% & 13.3\% $\pm$ 8.6\% & 5.9\%$\pm$4.1\% & 22.8\% $\pm$ 4.5\% & 16.7\% $\pm$ 7.0\%\\
\bottomrule
\end{tabular}}
\end{center}
\vspace{-0.3cm}
\label{tbl:gym}
\normalsize
\end{table*}

\begin{table*}[ht]
\vspace*{-0.3cm}
\caption{\footnotesize Results for multi-task imaged-based robotic manipulation domains in \citep{yu2021conservative}. Numbers are averaged across 3 seeds, $\pm$ the 95$\%$ confidence interval. UDS outperforms No Sharing in 7 out of 10 tasks as well as the average task performance, while performing comparably to Sharing All. CDS+UDS further improves the performance of UDS and outperforms No Sharing in all of the 10 tasks.
}
\small{
\begin{center}
\begin{tabular}{l|rrr|rr}
\toprule
\textbf{Task Name} & \textbf{CDS+UDS}& \textbf{UDS} & \textbf{No Sharing} & \textbf{CDS (oracle)} & \textbf{Sharing All (oracle)}\\ \midrule
\texttt{lift-banana} & \textbf{55.9\%}$\pm$11.7\% & 48.6\%$\pm$5.1\% & 20.0\%$\pm$6.0\% &  \textbf{53.1\%}$\pm$3.2\% & 41.8\%$\pm$4.2\%\\
\texttt{lift-bottle} & \textbf{72.9\%}$\pm$12.8\% & 58.1\%$\pm$3.6\%& 49.7\%$\pm$8.7\% & \textbf{74.0\%}$\pm$6.3\% & 60.1\%$\pm$10.2\%\\
\texttt{lift-sausage} & \textbf{74.3\%}$\pm$8.3\%  & 66.8\% $\pm$ 2.7\%  & 60.9\%$\pm$6.6\% & \textbf{71.8\%}$\pm$3.9\% & 70.0\%$\pm$7.0\%\\
\texttt{lift-milk}& 73.5\%$\pm$6.7\% & \textbf{74.5\%}$\pm$2.5\% & 68.4\%$\pm$6.1\% & \textbf{83.4\%}$\pm$5.2\% & 72.5\%$\pm$5.3\%\\

\texttt{lift-food} & \textbf{66.3\%}$\pm$8.3\% & 53.8\%$\pm$8.8\%  & 39.1\%$\pm$7.0\% & \textbf{61.4\%}$\pm$9.5\% & 58.5\%$\pm$7.0\%\\
\texttt{lift-can} & \textbf{64.9\%}$\pm$7.1\%  & 61.0\%$\pm$6.8\%  & 49.1\%$\pm$9.8\% & \textbf{65.5\%}$\pm$6.9\% & 57.7\%$\pm$7.2\%\\
\texttt{lift-carrot} & \textbf{84.1\%}$\pm$3.6\% & 73.4\%$\pm$5.8\% & 69.4\%$\pm$7.6\% & \textbf{83.8\%}$\pm$3.5\% & 75.2\%$\pm$7.6\%\\
\texttt{place-bowl} & \textbf{83.4\%}$\pm$3.6\%  & 77.6\%$\pm$1.6\%  & 80.3\%$\pm$8.6\% & \textbf{81.0\%}$\pm$8.1\% & 70.8\%$\pm$7.8\%\\
\texttt{place-plate} & \textbf{86.2\%}$\pm$1.8\%  & 78.7\%$\pm$2.2\%  & 86.1\%$\pm$7.7\% & \textbf{85.8\%}$\pm$6.6\% & 78.7\%$\pm$7.6\%\\
\texttt{place-divider-plate} & \textbf{89.0\%}$\pm$2.2\%  & 80.2\%$\pm$2.2\%  & 85.0\%$\pm$5.9\% & \textbf{87.8\%}$\pm$7.6\% & 79.2\%$\pm$6.3\%\\
\CC \textbf{average} & \CC \textbf{75.0\%}$\pm$3.3\%  & \CC 67.3\%$\pm$0.8\% & \CC 60.8\%$\pm$7.5\% & \CC \textbf{74.8\%} $\pm$6.4\% & \CC 66.4\%$\pm$7.2\%\\
\bottomrule
\end{tabular}
\end{center}
\vspace{-0.7cm}
\label{tbl:mtopt}
}
\end{table*}

\textbf{Comparisons.} 
We evaluate multiple approaches alongside with \textbf{UDS} and \textbf{CDS+UDS}:
\textbf{No Sharing}, which uses the labeled data only without using any of the unlabeled data, \textbf{Reward Predictor}, which is trained via supervised classification or regression to directly predict the reward in sparse and dense reward settings respectively, \textbf{VICE}~\citep{fu2018variational} and \textbf{RCE}~\citep{eysenbach2021replacing}, inverse RL methods only applicable to sparse reward settings, that learn either a reward or Q-function classifier from both the labeled data and unlabeled data (treated as negatives) and then annotate the unlabeled data with the learned classifier, 
In the multi-task setting, we modify \textbf{VICE} and \textbf{RCE} accordingly by extracting transitions with reward labels equal to 1 and treating these datapoints as positives to learn the classifier for each task. We also train \textbf{No Sharing}, \textbf{Reward Predictor}, \textbf{VICE} and \textbf{RCE}, but adapt them to the offline setting using CQL, i.e. the same underlying offline RL method as in \textbf{UDS} and \textbf{CDS+UDS}.
For more details on experimental set-up and hyperparameter settings, please see Appendix~\ref{app:details}. 
\rebuttal{We also include evaluations of our methods under different quality of the relabeled data in Appendix~\ref{app:data_quality} and comparisons to model-based offline RL approaches in Appendix~\ref{app:mbrl} and prior methods~\citep{yang2021representation,yang2021trail} that leverage unlabeled data for representation learning instead of sharing directly in Appendix~\ref{app:pretrained_reps}.}

\textbf{Datasets and tasks.} To answer questions (1) and (2), we perform empirical evaluations on two state-based single-task locomotion datasets. To answer question (3), we further evaluate all methods on three state-based multi-task datasets that consist of robotic manipulation, navigation and locomotion domains respectively and one image-based multi-task manipulation dataset from prior work \citep{yu2021conservative}.

\begin{figure}[ht]
    \centering
    \includegraphics[width=0.35\textwidth]{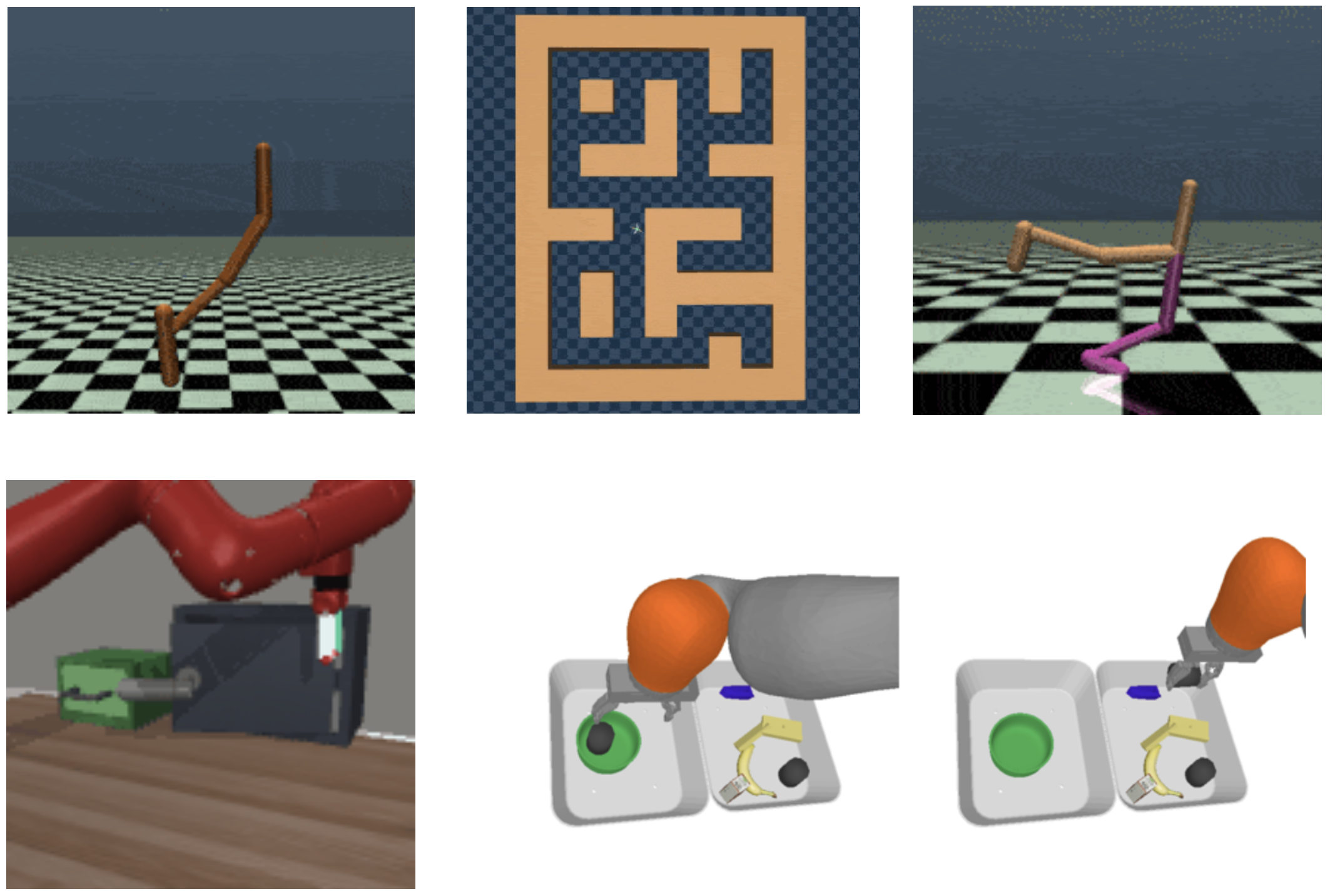}
    \vspace{-0.33cm}
    \caption{\footnotesize  Environments (from left to right):  single-task hopper, single-task and multi-task AntMaze, multi-task walker, multi-task Meta-World, and multi-task vision-based pick-place tasks.}
    \vspace{-0.35cm}
    \label{fig:env}
\end{figure}

\noindent \textbf{Single-task domains \& datasets.} We adopt two environments: hopper and the AntMaze from the D4RL benchmark~\citep{fu2020d4rl} for evaluation in the single-task setting. For the hopper environment, we consider two scenarios where we have 10k labeled data from the {hopper-expert} and 1M unlabeled transitions from the {hopper-random} datasets and {hopper-medium} datasets respectively. This setup is resembles practical problems where unlabeled data is of low-quality and even, irrelevant to the target task.
For AntMaze, following the setup in \cite{yang2021trail}, we use 10k expert transitions as the labeled data and the entire datasets of {large-play} and {medium-play} as the unlabeled data.

\begin{table*}[ht]
\vspace{-0.3cm}
\caption{\footnotesize We perform an empirical analysis on the single-task environment \texttt{hopper} from D4RL~\citep{fu2020d4rl} benchmark to test the sensitivity of UDS under the data quality and data coverage for both the labeled task data and unlabeled data. The numbers are averaged over \final{ten} random seeds. We bold the best method without true reward relabeling. UDS outperforms No Sharing in 5 out of 6 settings while achieving competitive performances compared to Sharing All in 5 out of 6 settings.
CDS+UDS is able to further improve over UDS significantly in such a setting and also outperforms UDS in 5 out of 6 settings in general.
}
\vspace{-0.3cm}
\begin{center}
\resizebox{\textwidth}{!}{\begin{tabular}{l|l|l|rrr|rr}
\toprule
& & & & & & \multicolumn{2}{c}{\textbf{Oracle Methods}} \\
 \textbf{Environment} & \textbf{Labeled data type / size} & \textbf{Unlabeled data type / size}   & \textbf{CDS+UDS}   & \textbf{UDS}         & \textbf{No Sharing} & \textbf{CDS} & \textbf{Sharing All} \\
 \midrule
& \textbf{(a)} expert / 10k transitions & random / 1M transitions & \textbf{82.1}  & 78.8 & 77.1 & 83.3 & 86.1\\
& \textbf{(b)} expert / 10k transitions & medium / 1M transitions   & \textbf{78.1} & 64.8 & 77.1 & 82.5 & 64.6\\
& \textbf{(c)} expert / 10k transitions & expert / 990k transitions   & 106.1 & \textbf{108.4} & 77.1 & 106.6 & 112.3\\
 D4RL hopper  & \textbf{(d)} medium / 10k transitions & random / 1M transitions & \textbf{33.2}  & 9.9 & 28.7 & 41.2 & 38.9\\
& \textbf{(e)} medium / 10k transitions & expert / 1M transitions & \textbf{108.9}   & 106.7 & 28.7 & 111.1 & 107.5\\
& \textbf{(f)} random / 10k transitions & medium / 1M transitions & \textbf{63.5}  & 47.1 & 9.6 & 92.3 & 69.8\\
& \textbf{(g)} random / 10k transitions & expert / 1M transitions  & \textbf{101.2} & 95.9 & 9.6 & 110.9 & 102.8\\
\bottomrule
\end{tabular}}
\end{center}
\vspace{-0.7cm}
\label{tbl:single_task_analysis}
\normalsize
\end{table*}

\noindent \textbf{Multi-task domains \& datasets.} We consider several multi-task domains. The first set of domains are from prior work~\citep{yu2021conservative}: \textbf{(i)} the Meta-World~\citep{yu2020meta} domain, which consists of four tasks of opening and closing doors and drawers; and \textbf{(ii)} the Antmaze~\citep{fu2020d4rl} domain, which consists of mazes of two sizes (medium and large) with 3 and 7 tasks corresponding to different goal positions respectively. We also evaluate on the multi-task walker environment with dense rewards, which we will discuss in detail in Appendix~\ref{app:dense_reward}.
In addition, we test multi-task offline RL methods with UDS in the multi-task visual manipulation domain, which provides a realistic scenario, of the sorts we will encounter in robotic settings in practice. In this domain, there are 10 tasks with different combinations of object-oriented grasping, with 7 objects (banana, bottle, sausage, milk box, food box, can and carrot), as well as placing the picked objects onto one of three fixtures (bowl, plate and divider plate). For domains except the walker environment, we use binary rewards, where $1$ denotes the successful completion of the task and $0$ corresponds to failure. We also use the datasets used in \cite{yu2021conservative} for all domains. For details, see Appendix~\ref{app:env_data_details}.

\vspace{-0.1cm}
\subsection{Results of Empirical Evaluations}
\label{sec:results}
\vspace{-0.1cm}
\noindent \textbf{Results of Question (1) and (2).} We evaluate each method on the single-task hopper and AntMaze domains. As shown in Table~\ref{tbl:single_task}, we find that UDS outperforms No Sharing in 3 out of the 4 settings and reward learning methods in all the settings. We hypothesize that reward learning methods underperform because reward predictors are unable to achieve reasonable generalization ability in the limited labeled data setting. Note that UDS even achieves competitive performance as the oracle Sharing All method. Furthermore, an optimized reweighting scheme, i.e., CDS+UDS, is able to improve over UDS in each case, including cases where UDS fails to improve over No Sharing, indicating a large reward bias. These results testify to the effectiveness of optimizing for reweighting when dealing with unlabeled data. Note that on the AntMaze domain, CDS+UDS performs comparably to the recent approach~\citep{yang2021trail} that performs representation learning on the offline dataset first and then run offline training leveraging the learned representation. CDS+UDS also outperforms all the prior methods in offline training with learned representations discussed in \cite{yang2021trail}.

\noindent \textbf{Results of Question (3).} Observe in Table~\ref{tbl:gym} that UDS outperforms na\"ive multi-task offline RL without data sharing and reward learning methods, suggesting that leveraging unlabeled data with our simple method UDS can boost offline RL performance in both multi-task manipulation and navigation domains.
Since reward learning approaches obtain similar or worse results compared to No Sharing, which we suspect could be due to erroneous predictions on unseen transitions in the multi-task data, we only compare our methods to No Sharing and the oracle methods in the image-based experiments. As shown in Table~\ref{tbl:mtopt}, UDS outperforms No Sharing in 7 out of 10 tasks as well as the average task performance by a significant margin. Therefore, UDS is able to effectively leverage unlabeled data from other tasks and achieves potentially surprisingly strong results compared to more sophisticated methods that handle unlabeled offline data. We also find that CDS+UDS further improves upon the performance of UDS, which indicates that optimizing for reweighting via CDS+UDS can actually work well.

Finally, note that, CDS+UDS and UDS attain performance competitive with their oracle counterparts, CDS and Sharing All, that assume access to full reward information, both in Tables~\ref{tbl:gym} \& \ref{tbl:mtopt}. This is potentially very surprising, and one could hypothesize that this might be because most transitions in the unlabeled dataset were actually failures, and hence, were labeled correctly by UDS. However, to the contrary, our diagnostic analysis in Table~\ref{tbl:data_quality}, Appendix~\ref{app:data_quality} reveals that the unlabeled data \emph{does not} consist entirely of failures; in fact, around 60\% of the unlabeled data succeeds at the task of interest, indicating that the rewards are wrong for more than half the unlabeled data. In spite of this, UDS and CDS+UDS perform well. This indicates that the simple UDS approach can be effective in removing the dependence of functional form of reward functions, which is often costly, without much loss in the performance and CDS+UDS can boost performance by reducing the bias.

\subsection{Empirical Analysis of UDS and CDS+UDS}
\label{sec:empirical_analysis}

To answer question (4), we analyze UDS and CDS+UDS on several offline RL problems designed to test robustness and sensitivity to the data coverage and the data quality on the single-task hopper domain. To create these instances, we consider 7 different combinations of data quality ({hopper-expert}, {hopper-medium} or {hopper-random}) and amount of labeled and unlabeled data labeled as \textbf{(a)}-\textbf{(g)} in Table~\ref{tbl:single_task_analysis}. In Appendix~\ref{app:mw_analysis}, we present results for a similar analysis in the multi-task Meta-World setting. We evaluate UDS, CDS+UDS and No Sharing in each case, and also present results for CDS and Sharing All approaches which assume access to the actual reward, for understanding. \final{Additionally, we conduct an ablation that compares UDS and CDS+UDS to reward learning methods in settings where the labeled data size and quality are varied, which is included in Appendix~\ref{app:reward_learning_ablation}.}

When the labeled data is of expert-level, adding unlabeled random or medium data (cases \textbf{a} and \textbf{b} in Table~\ref{tbl:single_task_analysis}) should only increase coverage, since the labeled data only consists of expert transitions. Moreover, the reward bias induced due to incorrect labeling of rewards on the medium unlabeled data should not hurt, since the 10k expert transitions retain their correct labels, and the medium/random data should only serve as negatives, provided the annotated rewards on this data are worse than the rewards in the expert data. Therefore, we expect the benefits of coverage to outweigh any reward bias; as seen in Table~\ref{tbl:single_task_analysis}, we find that UDS indeed helps compared to No Sharing. In particular, in those two cases, UDS approaches the oracle Sharing All method. 

Since reward bias is exacerbated when the unlabeled data is of higher quality and present in large amounts (cases \textbf{e, f, g}), it is reasonable to surmise that UDS will perform worse in such scenarios. However, on the contrary, in settings: random labeled data with medium or expert unlabeled data and medium labeled data with expert unlabeled data, we find that even though the rewards on the unlabeled transitions are incorrect, the addition of unlabeled data into training improves performance. This is because a higher quality unlabeled data improves the performance of the effective behavior policy, thereby improving performance for a conservative offline RL method. This result validates Theorem~\ref{prop:uds_ours}. 
We also conduct an ablation that varies the amount of unlabeled data in Appendix~\ref{app:unlabeled_dataset_size_analysis}. This ablation shows that the benefit of UDS reduces as the we reduce the amount of unlabeled data, which confirms our theory in Section~\ref{sec:uds_theory}.

In the case where labeled data is medium and unlabeled data is random (case \textbf{d}), we find that UDS hurts compared to No Sharing. This is because the labeled data as well as unlabeled data are both low-medium quality and medium data already provides decent coverage (not as high as random data, but not as low as expert data). Therefore, in this case, we believe that the addition of unlabeled data neither provides trajectories of good quality that can help improve performance, nor does it significantly improve coverage, and only hurts by incurring reward bias. We therefore believe that UDS may not help in such cases where the coverage does not improve, and added data is not so high quality, which also agrees with our theoretical analysis. Furthermore, in the case where the labeled and unlabeled data are both expert (case \textbf{c}), UDS performs close to oracle methods CDS and Sharing All, which confirms the insights derived from Section~\ref{sec:uds_theory} that UDS does not introduce bias when the labeled and unlabeled data have the same distribution and hence should perform well.

Finally, note that CDS+UDS is able to improve over UDS in 5 out of 6 settings in hopper, suggesting that CDS+UDS gains benefit from reducing the reward bias as well as the sampling error shown in Section~\ref{sec:method}.

\section{Discussion}

In this paper, we study the problem of leveraging unlabeled data in offline RL where we find that a simple method UDS that relabels unlabeled data with zero rewards is surprisingly effective in various single-task and multi-task offline RL domains. We provide both theoretical and empirical analysis of UDS to study conditions where it works and does not work. Furthermore, we show that by utilizing the optimized reweighting strategy, reward bias introduced in UDS can be reduced and the policy performance bound is improved in our theoretical analysis, which is also verified empirically. We believe that our analysis justifies the effectiveness of such simple methods for using unlabeled data in offline RL and shed light on directions for future work that can better control the reward bias and enjoy better policy performance.

\section*{Acknowledgements}
We thank Kanishka Rao, Julian Ibarz, other members of IRIS at Stanford, RAIL at UC Berkeley and Robotics at Google and Google Research and anonymous reviewers for valuable and constructive feedback on an early version of this manuscript. This research was funded in part by Google, ONR grants N00014-21-1-2685 and N00014-21-1-2838, Intel Corporation and the DARPA Assured Autonomy Program. CF is a CIFAR Fellow in the Learning in Machines and Brains program. 

\bibliography{references}
\bibliographystyle{icml2022}

\newpage
\appendix
\onecolumn

\section{\rebuttal{Proofs for Theoretical Analysis of UDS and Optimized Reweighting Schemes}}
In this section, we will theoretically analyze UDS and other reweighting schemes to better understand when these approaches can perform well. We will first discuss our notation, then present our theoretical results, then discuss the intuitive explanations of these results, and finally, provide proofs of the theoretical results.

\subsection{\rebuttal{Notation and Assumptions}}
\label{app:notation}
Let $\pi_\beta(\ba|\bs)$ denote the behavior policy of the dataset. The dataset, $\mathcal{D}$ 
is generated from the marginal state-action distribution of $\pi_\beta$, i.e., $\mathcal{D} \sim d^{\pi_\beta}(\bs) \pi_\beta(\ba|\bs)$. We define $d^{\pi}_{\mathcal{D}}$ as the state marginal distribution introduced by the dataset $\mathcal{D}$ under $\pi$. For our analysis, we will abstract conservative offline RL algorithms into a generic constrained policy optimization problem~\citep{kumar2020conservative}:
\begin{align}
\label{eqn:generic_offline_rl}
    \pi^*(\ba|\bs) := \arg \max_{\pi}~~ \widehat{J}_{\mathcal{D}}(\pi) - \frac{\alpha}{1 - \gamma} D(\pi, \pi_\beta).
\end{align}  
$J_{\mathcal{D}}(\pi)$ denotes the average return of policy $\pi$ in the empirical MDP induced by the transitions in the dataset, and $D(\pi, \pi_\beta)$ denotes a divergence measure (e.g., KL-divergence~\citep{jaques2019way,wu2019behavior}, MMD distance~\citep{kumar2019stabilizing} or $D_{\text{CQL}}$~\citep{kumar2020conservative}) between the learned policy $\pi$ and the behavior policy $\pi_\beta$ computed in expectation over the marginal state-action distribution induced by the policy in the empirical MDP induced by the dataset:
\begin{align*}
    D(\pi, \pi_\beta) = \mathbb{E}_{\bs \sim \widehat{d}^{\pi}_\mathcal{D}}\left[ D(\pi(\cdot|\bs), \pi_\beta(\cdot|\bs) \right].
\end{align*}
Let $D_\text{CQL}(p, q)$ denote the following distance between two distributions $p(\bx)$ and $q(\bx)$ with equal support $\mathcal{X}$:
\begin{equation*}
    D_\text{CQL}(p, q) := \sum_{\bx \in \mathcal{X}} p(\bx) \left(\frac{p(\bx)}{q(\bx)} - 1 \right).
\end{equation*}
Unless otherwise mentioned, we will drop the subscript ``CQL'' from $D_\text{CQL}$ and use $D$ and $D_\text{CQL}$ interchangeably. Prior works~\citep{kumar2020conservative,yu2021conservative} have shown that the optimal policy $\pi^*$ that optimizes Equation~\ref{eqn:generic_offline_rl} attains a high probability safe-policy improvement guarantee, i.e., $J(\pi^*) \geq J(\pi_\beta) - \zeta$, where $\zeta$ is:
\begin{align}
    \label{eqn:single_task_guarantee}
    \zeta =  \mathcal{O}\left(\frac{1}{(1 - \gamma)^2}\right) \mathbb{E}_{\bs \sim d^{\pi^*}_{\mathcal{D}}}\left[\sqrt{\frac{D_{\text{CQL}}(\pi^*, \pi_\beta)(\bs) + 1}{|\mathcal{D}(\bs)|}} \right] - \frac{\alpha}{1 - \gamma} D(\pi^*, \pi_\beta).
\end{align}
The first term in Equation~\ref{eqn:single_task_guarantee} corresponds to the decrease in performance due to sampling error and this term is high when the single-task optimal policy $\pi^*_i$ visits rarely observed states in the dataset $\mathcal{D}_i$ and/or when the divergence from the behavior policy $\pi_\beta$ is higher under the states visited by the single-task policy $\bs \sim d^{\pi^*_i}_{\mathcal{D}_i}$. We will show that UDS and CDS+UDS enjoy safe policy improvement. In our analysis, we assume $r(\bs, \ba) \in [0, 1]$. Finally, as discussed in Section~\ref{sec:prelim}, let $\mathcal{D}^\mathrm{eff}_i$ denote the relabeled dataset for task $i$, which includes both $\mathcal{D}_i$ and the transitions from other tasks relabeled with a $0$ reward.

\textbf{Assumptions.} To prove our theoretical results, following prior work~\citep{kumar2020conservative,yu2021conservative} we assume that the empirical rewards and dynamics concentrate towards their mean.
\begin{assumption}
\label{assumption:conc}
    $\forall~ \bs, \ba$, the following relationships hold with high probability, $\geq 1 - \delta$
    \begin{equation*}
        |\widehat{r}(\bs, \ba) - r(\bs, \ba)| \leq \frac{C_{r, \delta}}{\sqrt{|\mathcal{D}(\bs, \ba)|}}, ~~~ ||\widehat{P}(\bs'|\bs, \ba) - P(\bs'|\bs, \ba)||_{1} \leq \frac{C_{P, \delta}}{\sqrt{|\mathcal{D}(\bs, \ba)|}}.
    \end{equation*}
\end{assumption}
Similar to prior work~\citep{kumar2020conservative,yu2021conservative}, we also make a coverage assumption, i.e., we assume that each state-action pair is observed in the dataset $\mathcal{D}$, but the rewards and transition dynamics are stochastic, so, the occurrence of each state-action pair does not trivially imply good performance. To relax this assumption, we can extend our analysis to function approximation (e.g., linear function approximation~\citep{duan2020minimax}), where such a coverage assumption is only required on all directions of the feature space, and not all state-action pairs. This would not significantly change the analysis, and hence we opt for the simple but, at the same time, an illustrative analysis in a tabular setting here.

\subsection{\rebuttal{Performance Guarantee for UDS}: Proof for Theorem~\ref{prop:uds_ours}}
\label{proof:uds_proof}
We first restate Proposition~\ref{prop:uds_ours} below for convenience, and we then provide a proof of the result.
\begin{theorem}[\textbf{Policy improvement guarantee for UDS}; restated.] 
\label{prop:uds_ours_restated}
Let $\pi^*_\text{UDS}$ denote the policy learned by UDS, and let $\pi^\mathrm{eff}_\beta(\ba|\bs)$ denote the behavior policy for the combined dataset $\mathcal{D}^\mathrm{eff}$. Then with high probability $\geq 1 - \delta$, $\pi^*_\text{UDS}$ is a safe policy improvement over $\pi_\beta^\mathrm{eff}$, i.e.,
\begin{align*}
& J(\pi^*_\text{UDS}) \geq J(\pi_\beta^\mathrm{eff}) - \zeta_\text{err} +  \underbrace{\frac{\alpha}{1 - \gamma} D(\pi^*_\text{UDS}, \pi^\mathrm{eff}_\beta)}_{\text{(c): policy improvement}},\\
 ~\text{where:~~}& \zeta_\text{err} = \underbrace{\frac{\sum_{\bs, \ba} \left( \widehat{d}^{\behavior^\mathrm{eff}}(\bs, \ba) - \widehat{d}^{\pi^*_\text{UDS}}(\bs, \ba)\right)  \cdot (1 - f(\bs, \ba)) \cdot r(\bs, \ba)}{1 - \gamma}}_{ \text{(a): reward bias}} + \underbrace{\mathcal{O}\left(\frac{\gamma}{(1 - \gamma)^2}\right) \left[\sqrt{\frac{D_{\text{CQL}}(\pi^*_\text{UDS}, \pi^\mathrm{eff}_\beta)(\bs)}{|\mathcal{D}^\mathrm{eff}(\bs)|}} \right]}_{\text{(b): sampling error}},
\end{align*}
where we use the notation $f(\bs, \ba) := \frac{|\mathcal{D}_\mathrm{L}(\bs, \ba)|}{|\mathcal{D}^\mathrm{eff} (\bs, \ba)|}$.
\end{theorem}

\begin{proof}
We start with the loss decomposition of the improvement of the learned policy relative to the behavior policy with the affine transformation $g$:  
\begin{align*}
    J(\pi) - J(\pi_\beta) := \underbrace{J(\pi) - \widehat{J}(\pi)}_{(i)} + \underbrace{\widehat{J}(\pi)  - \widehat{J}(\pi_\beta)}_{(ii)} + \underbrace{\widehat{J}(\pi_\beta) - {J}(\pi_\beta)}_{(iii)}.
\end{align*}

Now we will discuss how to bound each of the terms: terms (i) and (ii) correspond to the divergence between the empirical policy return and the actual return. While usually, this difference depends on the sampling error and distributional shift, in our case, it additionally depends on the reward bias induced on the unlabeled data and the transformation $g$. We first discuss the terms that contribute to this reward bias.

\textbf{Bounding the reward bias.} Denote the effective reward of a particular transition $(\bs, \ba, r, \bs') \in \mathcal{D}^\mathrm{eff}$, as $\widehat{r}^\mathrm{eff}$, which considers contributions from both the reward $\widehat{r}(\bs, \ba)$ observed in dataset $\mathcal{D}_\mathrm{L}$, and the contribution of $0$ reward from the relabeled dataset:
\begin{align}
\label{eqn:relabeled_reward}
    \widehat{r}^\mathrm{eff}(\bs, \ba) = \frac{|\mathcal{D}(\bs, \ba)| \cdot \widehat{r}(\bs, \ba) + |\mathcal{D}^\mathrm{eff}(\bs, \ba) \setminus \mathcal{D}(\bs, \ba)| \cdot 0}{|\mathcal{D}^\mathrm{eff}(\bs,\ba)|}
\end{align}
Define $f(\bs, \ba) := \frac{|\mathcal{D}(\bs, \ba)|}{|\mathcal{D}^\mathrm{eff}(\bs, \ba)|}$ for notation compactness. Equation~\ref{eqn:relabeled_reward} can then be used to derive the following difference against the true rewards:
\begin{align}
\label{eqn:upper_bound_reward_bias}
    \widehat{r}^\mathrm{eff}(\bs, \ba) &- r(\bs, \ba) = f(\bs, \ba) \left( \widehat{r}(\bs, \ba) - r(\bs, \ba) \right) + (1 - f(\bs, \ba)) \cdot (0 - r(\bs, \ba)) \\
    &\leq f(\bs, \ba) \cdot \frac{C_{r, \delta}}{\sqrt{|\mathcal{D}(\bs, \ba)|}} - (1 - f(\bs, \ba)) \cdot r(\bs, \ba),
\end{align}
where the last step follows from the fact that the ground-truth reward $r(\bs, \ba) \in [0, 1]$. Now, we lower bound the reward bias as follows:
\begin{align}
    \label{eqn:lower_bound_reward_bias}
        \widehat{r}^\mathrm{eff}(\bs, \ba) + &- r(\bs, \ba) = f(\bs, \ba) \cdot \left(\widehat{r}(\bs, \ba) - r(\bs, \ba) \right) + (1 - f(\bs, \ba)) \cdot (- r(\bs, \ba)) \\
        &\geq - f(\bs, \ba) \cdot \frac{C_{r, \delta}}{\sqrt{|\mathcal{D}(\bs, \ba)|}} - (1 - f(\bs, \ba)) \cdot r(\bs, \ba), \nonumber
\end{align}
where the last step follows from the fact that $r(\bs, \ba) \leq 1$.

\textbf{Upper bounding $\widehat{J}(\pi) - J(\pi)$.} Next, using the upper and lower bounds on the reward bias, we now derive an upper bound on the difference between the value of a policy computed under the empirical MDP and the actual MDP. To compute this difference, we follow the following steps
\begin{align}
\label{eqn:decomposition}
    &\widehat{J}(\pi) - J(\pi) = \frac{1}{1 - \gamma} \sum_{\bs, \ba} \left(\widehat{d}^\pi_{\mathcal{D}^\mathrm{eff}}(\bs) \pi(\ba|\bs) \widehat{r}^\mathrm{eff}(\bs, \ba) - d^\pi(\bs) \pi(\ba|\bs) r(\bs, \ba)\right)\\ 
    &= \frac{1}{1 - \gamma} \underbrace{\sum_{\bs,\ba} \widehat{d}^\pi_{\mathcal{D}^\mathrm{eff}}(\bs) \pi(\ba|\bs) \left(\widehat{r}^\mathrm{eff}(\bs, \ba) - r(\bs, \ba)\right)}_{:= \Delta_1} + \frac{1}{1 - \gamma} \underbrace{\sum_{\bs, \ba}\left(\widehat{d}^\pi_{\mathcal{D}^\mathrm{eff}}(\bs) - d^\pi(\bs)\right) \pi(\ba|\bs) r(\bs, \ba)}_{:= \Delta_2} \nonumber
\end{align}
Following \citet{kumar2020conservative} (Theorem 3.6), we can bound the second term $\Delta_2$ using:
\begin{align}
\label{eqn:delta_2_bound}
    \left|\Delta_2\right| \leq \frac{\gamma C_{P, \delta}}{1 - \gamma} \mathbb{E}_{\bs \sim \widehat{d}^\pi_{\mathcal{D}^\mathrm{eff}}(\bs)}\left[ \frac{\sqrt{|\mathcal{A}|}}{\sqrt{|\mathcal{D}^\mathrm{eff}(\bs)|}} \sqrt{ D(\policy, \widehat{\pi}^\mathrm{eff}_\beta)(\bs) + 1} \right].
\end{align}
To upper bound $\Delta_1$, we utilize the reward upper bound from Equation~\ref{eqn:upper_bound_reward_bias}:
\begin{align}
    \Delta_1 &= \sum_{\bs} \widehat{d}^\pi_{\mathcal{D}^\mathrm{eff}}(\bs) \left( \sum_{\ba} \pi(\ba|\bs) \left(\widehat{r}^\mathrm{eff}(\bs, \ba) - r(\bs, \ba)\right) \right) \\
    &\leq \underbrace{\sum_\bs \widehat{d}^\pi_{\mathcal{D}^\mathrm{eff}}(\bs) \sum_{\ba} f(\bs, \ba) \frac{C_{r, \delta}}{\sqrt{|\mathcal{D}(\bs)|}} \frac{\pi(\ba|\bs)}{\sqrt{\hatbehavior(\ba|\bs)}}}_{\:= \Delta'_1} - \underbrace{\sum_{\bs, \ba}  \widehat{d}^\pi_{\mathcal{D}^\mathrm{eff}}(\bs) \pi(\ba|\bs) \left[ (1 - f(\bs, \ba)) \cdot r(\bs, \ba) \right]}_{:=\Delta_4}.  
\end{align}
Combining the results so far, we obtain, for any policy $\pi$:
\begin{align}
\label{eqn:sampling_error_upper_bound}
    J(\pi) &\geq \widehat{J}(\pi) - \frac{|\Delta_2|}{1 - \gamma} - \frac{|\Delta_1'|}{1 - \gamma} + \frac{\Delta_4}{1 - \gamma}.
\end{align}

\textbf{Lower bounding $\widehat{J}(\pi) - J(\pi)$.} To lower bound this quantity, we follow the step shown in Equation~\ref{eqn:decomposition}, and lower bound the term $\Delta_2$ by using the negative of the RHS of Equation~\ref{eqn:delta_2_bound}, and lower bound $\Delta_1$ by upper bounding its absolute value as shown below:
\begin{align}
    &\Delta_1 =  \sum_{\bs} \widehat{d}^\pi_{\mathcal{D}^\mathrm{eff}}(\bs) \left( \sum_{\ba} \pi(\ba|\bs) \left(\widehat{r}^\mathrm{eff}(\bs, \ba) - r(\bs, \ba)\right) \right) \\
    &\geq \underbrace{\sum_\bs \widehat{d}^\pi_{\mathcal{D}^\mathrm{eff}}(\bs) \sum_{\ba} f(\bs, \ba) \frac{C_{r, \delta}}{\sqrt{|\mathcal{D}(\bs)|}} \frac{\pi(\ba|\bs)}{\sqrt{\hatbehavior(\ba|\bs)}}}_{\:= \Delta'_1} + \sum_{\bs} \widehat{d}^\pi_{\mathcal{D}^\mathrm{eff}}(\bs) \sum_{\ba} \pi(\ba|\bs) \cdot (1 - f(\bs, \ba)) r(\bs, \ba).  
\end{align}
This gives rise to the complete lower bound:
\begin{align}
\label{eqn:sampling_error_lower_bound}
   \widehat{J}(\pi) &\geq J(\pi) - \frac{|\Delta_2|}{1 - \gamma} - \frac{1}{1-\gamma} \sum_{\bs, \ba} \widehat{d}^\pi_{\mathcal{D}^\mathrm{eff}_i}(\bs) \pi(\ba|\bs) (1 - f(\bs, \ba)) r(\bs, \ba)  -  \frac{\Delta'_1}{1 - \gamma}.
\end{align}

\textbf{Policy improvement term (ii).} Finally, the missing piece that needs to be bounded is the policy improvement term (ii) in the decomposition of $J(\pi) - J(\pi_\beta)$. Utilizing the abstract form of offline RL (Equation~\ref{eqn:generic_offline_rl}, we note that term (ii) is lower bounded as:
\begin{align}
\label{eqn:lower_bound_on_improvement}
    \text{term (ii)} \geq \frac{\alpha}{1 - \gamma} D(\pi, \pi_\beta).
\end{align}

\textbf{Putting it all together.} To obtain the final expression of Proposition~\ref{prop:uds_ours}, we put all the parts together, and include some simplifications to obtain the final expression. The bound we show is relative to the effective behavior policy $\pi^\mathrm{eff}_\beta$. Applying Equation~\ref{eqn:sampling_error_lower_bound} for term (i) on policy $\pi$, Equation~\ref{eqn:lower_bound_on_improvement} for term (ii), and Equation~\ref{eqn:sampling_error_upper_bound} for the behavior policy $\pi^\mathrm{eff}_\beta$, we obtain the following:
\begin{align*}
    &J(\pi) - J(\pi_\beta^\mathrm{eff}) = J(\pi) - \widehat{J}(\pi) + \widehat{J}(\pi) - \widehat{J}(\pi_\beta^\mathrm{eff}) + \widehat{J}(\pi_\beta^\mathrm{eff}) - J(\pi_\beta^\mathrm{eff})\\
    &\geq -\frac{2 \gamma C_{P, \delta}}{(1 - \gamma)^2} \mathbb{E}_{\bs \sim \widehat{d}^\pi_{\mathcal{D}^\mathrm{eff}}(\bs)}\left[ \frac{\sqrt{|\mathcal{A}|}}{\sqrt{|\mathcal{D}^\mathrm{eff}(\bs)|}} \sqrt{ D(\policy, \widehat{\pi}^\mathrm{eff}_\beta)(\bs) + 1} \right] - \frac{2 C_{r, \delta}}{1 - \gamma} \mathbb{E}_{\bs, \ba \sim \widehat{d}^\pi_{\mathcal{D}^\mathrm{eff}}} \left[ \frac{f(\bs, \ba)}{\sqrt{|\mathcal{D}(\bs, \ba)|}} \right]\\
    &~~~~~~~~~~~~~~~~~~~~~~~ - \frac{1}{1 - \gamma} {\left( \mathbb{E}_{\bs, \ba \sim d^{\behavior^\mathrm{eff}}_{\mathcal{D}^\mathrm{eff}}}\left[\left(1 - f(\bs, \ba)\right) r(\bs, \ba) \right] \right) } + \frac{1}{1 - \gamma} \mathbb{E}_{\bs, \ba \sim \widehat{d}^\pi_{\mathcal{D}^\mathrm{eff}}} \left[ (1 - f(\bs, \ba)) \cdot r(\bs, \ba) \right]\\
    &~~~~~~~~~~~~~~~~~~~~~~~+ \frac{\alpha}{1 - \gamma} D(\pi, \pi_\beta^\mathrm{eff}).
\end{align*}
Note that in the second step above, we upper bound the quantities $\Delta_1'$ and $\Delta_2$ corresponding to $\pi_\beta^\mathrm{eff}$ with twice the expression for policy $\pi$. This is because the effective behavior policy $\pi^\mathrm{eff}_\beta$ consists of a mixture of the original behavior policy $\hatbehavior$ with the additional data, and thus the new effective dataset consists of the original dataset $\mathcal{D}_i$ as its part. Upper bounding it with twice the corresponding term for $\pi$ is a valid bound, though a bit looser, but this bound suffices for our interpretations. 

For our analysis purposes, we will define the suboptimality induced in the bound due to reward bias for a given $u$ and $v$ as:
\begin{align}
\label{eqn:reward_bias}
    \mathrm{RewardBias}(\pi, \pi^\mathrm{eff}_\beta) = -\frac{1}{1 - \gamma} \left[ \mathbb{E}_{\bs, \ba \sim \widehat{d}^\pi_{\mathcal{D}^\mathrm{eff}}} \left[ (1 - f(\bs, \ba)) \cdot r(\bs, \ba) \right] - \left(\mathbb{E}_{\bs, \ba \sim d^{\behavior^\mathrm{eff}}_{\mathcal{D}^\mathrm{eff}}}\left[\left(1 - f(\bs, \ba)\right) r(\bs, \ba) \right] \right) \right]
\end{align}
Thus, we obtain the desired bound in Proposition~\ref{prop:uds_ours}.
\end{proof}

\subsection{When is Reward Bias Small? Proof of Theorem \ref{prop:reward_bias_theorem}}
\label{sec:small_reward_bias}
Next, we wish to understand when the reward bias in Equation~\ref{eqn:reward_bias} is small. Concretely, we wish to search for effective behavior policies such that the dataset induced by them attains a small reward bias. Therefore we provide a proof for Theorem~\ref{prop:reward_bias_theorem} in this section.

\begin{proof}
We can express the reward bias as:
\begin{align*}
    \mathrm{RewardBias}(\pi, \pi^\mathrm{eff}_\beta) &:= - \frac{1}{1 - \gamma} \sum_{\bs, \ba} \left(\widehat{d}^\pi_{\mathcal{D}^\mathrm{eff}}(\bs, \ba) - \widehat{d}^{\behavior^\mathrm{eff}}_{\mathcal{D}^\mathrm{eff}}(\bs, \ba) \right) \cdot (1 - f(\bs, \ba)) \cdot r(\bs, \ba),
\end{align*}
Now, since our goal is to minimize the reward bias with respect to the effective behavior policy, we minimize the expression for suboptimality induced due to reward bias, shown above with respect to $\pi^\mathrm{eff}_\beta$. Before performing the differentiation step, we note the following simplification (we drop the $\mathcal{D}^\mathrm{eff}$ from the notation in $d^{\behavior^{\mathrm{eff}}}_{\mathcal{D}^\mathrm{eff}}$ to make the notation less cluttered below):
\begin{align}
    \min_{\pi^\mathrm{eff}_\beta}~~  \mathrm{RewardBias}(\pi, \pi^\mathrm{eff}_\beta) &:= \min_{\pi^\mathrm{eff}_\beta}~~ -\left(\widehat{J}(\pi) - \widehat{J}(\pi_\beta^\mathrm{eff}) \right) + \frac{1}{(1 - \gamma)} \sum_{\bs, \ba} r(\bs, \ba) f(\bs, \ba)  \left(\widehat{d}^\pi_{\mathcal{D}^\mathrm{eff}}(\bs, \ba) - d^{\behavior^\mathrm{eff}}_{\mathcal{D}^\mathrm{eff}}(\bs, \ba) \right) \nonumber\\
    &= \min_{\pi^\mathrm{eff}_\beta}~~ - \widehat{J}(\pi) + \widehat{J}(\pi^\mathrm{eff}_\beta) + \frac{1}{(1 - \gamma) |\mathcal{D}^\mathrm{eff}|}\sum_{\bs, \ba} |\mathcal{D}(\bs, \ba)| r(\bs, \ba) \left( \frac{\widehat{d}^\pi(\bs, \ba)}{\widehat{d}^{\pi_\beta^\mathrm{eff}}(\bs, \ba)} - 1 \right) \nonumber\\
    &= \min_{\pi^\mathrm{eff}_\beta}~~ \widehat{J}(\pi^\mathrm{eff}_\beta) +  \frac{1}{(1 - \gamma) |\mathcal{D}^\mathrm{eff}|}\sum_{\bs, \ba} |\mathcal{D}(\bs, \ba)| r(\bs, \ba) \left( \frac{\widehat{d}^\pi(\bs, \ba)}{\widehat{d}^{\pi_\beta^\mathrm{eff}}(\bs, \ba)} - 1 \right). \label{eq:final_Expr}
\end{align}
Now, since we can express the entire objective in Equation~\ref{eq:final_Expr} as a function of $\widehat{d}^{\behavior^\mathrm{eff}}$ since $\widehat{J}(\pi^\mathrm{eff}_\beta) = \sum_{\bs, \ba} \widehat{d}^{\behavior^\mathrm{eff}}(\bs, \ba) r(\bs, \ba)$, we can compute and set the derivative of Equation~\ref{eq:final_Expr} with respect to $\widehat{d}^{\behavior^\mathrm{eff}}(\bs, \ba)$ as 0 while adding down the constraints that pertain to the validity of $\pi$. This gives us:
\begin{equation*}
    \widehat{d}^\pi(\bs, \ba) \propto \sqrt{\frac{{|\mathcal{D}(\bs , \ba)| \cdot \widehat{d}^\pi(\bs, \ba)}}{|\mathcal{D}^\mathrm{eff}|}}. 
\end{equation*}
Substituting $|\mathcal{D}(\bs,\ba)| = \widehat{d}_\mathrm{L}(\bs, \ba) \cdot |\mathcal{D}_\mathrm{L}|$ we obtain the desired result.
\end{proof}

\subsection{When is The Bound in Theorem~\ref{prop:uds_ours} Tightest?: Proof of Theorem~\ref{thm:with_all_sources}}
\label{proof:all_sources}
In this section, we will formally state and provide a proof for Theorem~\ref{thm:with_all_sources}. To prove this result, we first compute an upper bound on the error terms (a) and (b) in Theorem~\ref{thm:with_all_sources}, and show that the optimized distribution shown in Theorem~\ref{thm:with_all_sources} emerges as a direct consequence of optimizing this bound. To begin, we compute a different upper bound on sampling error than the one used in Theorem~\ref{thm:with_all_sources}. Defining the sampling error for a policy $\pi$ as the difference in return in the original and the empirical MDPs $\Delta_\text{sampling} = \widehat{J}(\pi) - J(\pi)$, we obtain the following Lemma:

\begin{lemma}[Upper bound on sampling error in terms of $\widehat{d}^\pi(\bs, \ba)$]
We can upper bound sampling error term as follows: 
\begin{align}
\label{eqn:upper_bound_on_sampling}
    \Delta_\text{sampling} \leq \frac{\gamma C_{P, \delta}}{(1 - \gamma)^2 \sqrt{|\mathcal{D}|}} \sum_{\bs, \ba} \frac{\widehat{d}^\pi(\bs, \ba)}{\sqrt{\widehat{d}^{\pi_\beta^\mathrm{eff}}(\bs, \ba)}}.
\end{align}
\end{lemma}
\begin{proof}
For this proof, we will derive a bound on the sampling error, starting from scratch, but this time only in terms of the state-action marginals: 
\begin{align*}
    \Delta_{\text{sampling}} &= \frac{1}{1 - \gamma} \sum_{\bs, \ba} \left( \widehat{d}^\pi(\bs, \ba) - d^\pi(\bs, \ba) \right) \cdot r(\bs, \ba)
\end{align*}
Note that we can bound $\Delta_\text{sampling}$ by upper bounding the total variation between marginal state-action distributions in the empirical and actual MDPs, i.e., $\vert\vert \widehat{d}^\policy - d^\policy \vert\vert_1$, since $|r_M(\bs, \ba)| \leq R_{\max}$. and hence we bound the second term effectively. Our analysis is similar to \citet{achiam2017constrained}. Define, $G = (I - \gamma P^{\policy})^{-1}$ and $\bar{G} = (I - \gamma \widehat{P}^\policy)^{-1}$. Then,
\begin{equation*}
    \widehat{d}^\policy - d^\policy = (1 - \gamma) (\bar{G} - G) \rho,
\end{equation*}
where $\rho(\bs)$ is the initial state distribution. Then we can use the derivation in proof of Theorem 3.6 in \citet{kumar2020conservative} or Equation 21 from \citet{achiam2017constrained}, to bound this difference as
\begin{align*}
    \vert\vert d^\policy - \widehat{d}^\policy||_{1} &\leq \frac{\gamma}{1 - \gamma} \sum_{\bs} \widehat{d}^\policy(\bs) \frac{C_{P, \delta}}{\sqrt{|\mathcal{D}(\bs)|}} \sum_{\ba}  \frac{\policy(\ba|\bs)}{\sqrt{\behavior(\ba|\bs)}}\\
    &\leq \frac{\gamma C_{P, \delta}}{(1 - \gamma) \sqrt{|\mathcal{D}|}} \sum_{\bs, \ba} \frac{\widehat{d}^\pi(\bs, \ba)}{\sqrt{\widehat{d}^{\pi_\beta^\mathrm{eff}}(\bs, \ba)}}.
\end{align*}
Thus the sampling error can be bounded by:
\begin{align*}
    \Delta_\text{sampling} \leq \frac{\gamma C_{P, \delta}}{(1 - \gamma)^2 \sqrt{|\mathcal{D}|}} \sum_{\bs, \ba} \frac{\widehat{d}^\pi(\bs, \ba)}{\sqrt{\widehat{d}^{\pi_\beta^\mathrm{eff}}(\bs, \ba)}},
\end{align*}
which proves the lemma.
\end{proof}

\begin{theorem}[\textbf{Optimized reweighting unlabeled data}]
\label{thm:with_all_sources_restated}
The optimal effective behavior policy that maximizes a lower bound on $J(\pi^\mathrm{eff}_\beta) - \left[(a) + (b) \right]$ in Theorem~\ref{prop:uds_ours} satisfies: $d^{\widehat{\pi}_\beta^\mathrm{eff}}(\bs, \ba) = p^*(\bs, \ba)$, where,
\begin{align*}
    p^* = \arg\min_{p \in \Delta^{|\mathcal{S}||\mathcal{A}|}}&~ \sum_{\bs, \ba} C_1 \frac{\widehat{d}^\pi(\bs, \ba)}{\sqrt{p(\bs, \ba)}} + C_2 |d_\mathrm{L}(\bs, \ba)| \frac{\widehat{d}^\pi(\bs, \ba)}{p(\bs, \ba)}, 
\end{align*}
where $C_1$ and $C_2$ are universal positive constants that depend on the sizes of the labeled and unlabeled datasets are shown in Equation~\ref{eqn:c1_c2}.  
\end{theorem}

\begin{proof}
To prove this result, we will first simplify the expression containing all terms except the policy imnprovement term in Theorem~\ref{thm:with_all_sources}, so that we can then maximize the bound to obtain the statement of our theoretical statement. 
\begin{align*}
    (\bullet) &:= J(\pi^\mathrm{eff}_\beta) - \left[(a) + (b)\right].\geq  J(\pi^\mathrm{eff}_\beta) + \frac{1}{1 - \gamma} \sum_{\bs, \ba} \left(\widehat{d}^\pi(\bs, \ba) - \widehat{d}^{\behavior^\mathrm{eff}}(\bs, \ba) \right) \cdot (1 - f(\bs, \ba)) \cdot r(\bs, \ba) - \Delta_\text{sampling}\\
    & \geq J(\pi^\mathrm{eff}_\beta) + \frac{1}{1 - \gamma} \sum_{\bs, \ba} \left(\widehat{d}^\pi(\bs, \ba) - \widehat{d}^{\behavior^\mathrm{eff}}(\bs, \ba) \right) \cdot (1 - f(\bs, \ba)) \cdot r(\bs, \ba) - \frac{\gamma C_{P, \delta}}{(1 - \gamma)^2 \sqrt{|\mathcal{D}_\mathrm{eff}}|} \sum_{\bs, \ba} \frac{\widehat{d}^\pi(\bs, \ba)}{\sqrt{\widehat{d}^{\pi_\beta^\mathrm{eff}}(\bs, \ba)}} 
\end{align*}
First of all we can lower bound, $J(\pi^\mathrm{eff}_\beta)$ in terms of the return of $\pi^\mathrm{eff}_\beta)$ in the empirical MDP induced by the dataset $\mathcal{D}_\mathrm{eff}$ and an irreducible sampling error term, that grows as $\mathcal{O}\left(\sqrt{1/{|\mathcal{D}^\mathrm{eff}|}}\right)$ and does not depend on the distribution of state-action pairs in the effective dataset, $\widehat{d}^{\pi_\beta^\mathrm{eff}}(\bs, \ba)$, but only depends on its size. Using this result, and by performing algebraic manipulation, we can further simplify this as:
\begin{align*}
    (\bullet) &\geq \cancel{\widehat{J}(\pi^\mathrm{eff}_\beta)} + \widehat{J}(\pi) - \cancel{\widehat{J}(\pi^\mathrm{eff}_\beta)} - \frac{\sum_{\bs, \ba} |\mathcal{D}(\bs, \ba)| r(\bs, \ba) \left( \frac{\widehat{d}^\pi(\bs, \ba)}{\widehat{d}^{\pi_\beta^\mathrm{eff}}(\bs, \ba)} - 1 \right)}{(1 - \gamma) |\mathcal{D}^\mathrm{eff}|} - \frac{\gamma C_{P, \delta}}{(1 - \gamma)^2 \sqrt{|\mathcal{D}_\mathrm{eff}}|} \sum_{\bs, \ba} \frac{\widehat{d}^\pi(\bs, \ba)}{\sqrt{\widehat{d}^{\pi_\beta^\mathrm{eff}}(\bs, \ba)}}\\
    &+ \mathcal{O}\left(\sqrt{\frac{1}{|\mathcal{D}^\mathrm{eff}|}}\right)\\
    & \geq \widehat{J}(\pi) - \frac{1}{(1 - \gamma) |\mathcal{D}^\mathrm{eff}|}\sum_{\bs, \ba} |\mathcal{D}(\bs, \ba)| r(\bs, \ba) \left( \frac{\widehat{d}^\pi(\bs, \ba)}{\widehat{d}^{\pi_\beta^\mathrm{eff}}(\bs, \ba)} - 1 \right) - \frac{\gamma C_{P, \delta}}{(1 - \gamma)^2 \sqrt{|\mathcal{D}_\mathrm{eff}}|} \sum_{\bs, \ba} \frac{\widehat{d}^\pi(\bs, \ba)}{\sqrt{\widehat{d}^{\pi_\beta^\mathrm{eff}}(\bs, \ba)}}\\
    &+ \mathcal{O}\left(\sqrt{\frac{1}{|\mathcal{D}^\mathrm{eff}|}}\right)\\
    & \geq \widehat{J}(\pi) - \frac{1}{(1 - \gamma) |\mathcal{D}^\mathrm{eff}|}\sum_{\bs, \ba} |\mathcal{D}(\bs, \ba)| \left( \frac{\widehat{d}^\pi(\bs, \ba)}{\widehat{d}^{\pi_\beta^\mathrm{eff}}(\bs, \ba)} - 1 \right) - \frac{\gamma C_{P, \delta}}{(1 - \gamma)^2 \sqrt{|\mathcal{D}_\mathrm{eff}}|} \sum_{\bs, \ba} \frac{\widehat{d}^\pi(\bs, \ba)}{\sqrt{\widehat{d}^{\pi_\beta^\mathrm{eff}}(\bs, \ba)}} + \mathcal{O}\left(\sqrt{\frac{1}{|\mathcal{D}^\mathrm{eff}|}}\right),
\end{align*}
where the last inequality follows from the fact that $|r(\bs, \ba)| \leq 1$. Since $|\mathcal{D}^\mathrm{eff}|$ and $\widehat{d}^{\pi_\beta^\mathrm{eff}}$ are decoupled (one is the distribution; other is the size of the dataset), we can optimize over each of them independently, and hence, we find that the distribution that optimizes this bound is given by:
\begin{align*}
    p^* = \arg\min_{p \in \Delta^{|\mathcal{S}||\mathcal{A}|}}&~ \sum_{\bs, \ba} C_1 \frac{\widehat{d}^\pi(\bs, \ba)}{\sqrt{p(\bs, \ba)}} + C_2 |d_\mathrm{L}(\bs, \ba)| \frac{\widehat{d}^\pi(\bs, \ba)}{p(\bs, \ba)}, 
\end{align*}
where:
\begin{align}
\label{eqn:c1_c2}
    C_2 &:= \frac{|\mathcal{D}_\mathrm{L}|}{ (1 - \gamma) |\mathcal{D}_\mathrm{eff}|}, ~~~~~~~~ C_1 :=  \frac{\gamma C_{P, \delta}}{(1 - \gamma)^2 \sqrt{|\mathcal{D}_\mathrm{eff}}|}.
\end{align}
This proves Theorem~\ref{thm:with_all_sources}.
\end{proof}

\section{\rebuttal{Additional empirical analysis of the reason that CDS+UDS and UDS work}}
\label{app:empirical_analysis}

\rebuttal{In this section, we perform an empirical study on the Meta-World domain to better understand the reason that UDS and CDS+UDS work well. Our theoretical analysis suggests that UDS will help the most on domains with limited data or narrow coverage or low data quality. To test these conditions in practice, we perform empirical analysis on two domains as follows.

\subsection{\rebuttal{Meta-World Domains}}
\label{app:mw_analysis}

\begin{table*}[t!]
\centering
\resizebox{\textwidth}{!}{\begin{tabular}{l|l|r|r|r}
\toprule
\textbf{Environment} & \textbf{Dataset type / size} & \textbf{CDS+UDS} & \textbf{UDS} & No Sharing\\
\midrule
& expert / 2k transitions & \textbf{67.6\%} & 58.8\% & 31.3\%\\
Meta-World door open & medium / 2k transitions & 67.3\% & \textbf{74.2\%} & 27.6\%\\
& medium-replay / 152k transitions & \textbf{30.0\%} & 0.0\% & 14.8\%\\
\bottomrule
\end{tabular}}
\caption{\footnotesize \rebuttal{We perform an empirical analysis on the Meta-World \texttt{door open} task where we use varying data quality and dataset size target task \texttt{door open}. We share the same dataset from the other three tasks in the multi-task Meta-World environment, \texttt{door close}, \texttt{drawer open} and \texttt{drawer close} to the target task. The numbers are averaged over three random seeds. CDS+UDS and UDS are able to outperform No Sharing in most of the settings except that UDS fails to achieve non-zero success rate in the medium-replay dataset with a large number of transitions. Such results suggest that CDS+UDS and UDS are robust to the data quality of the target task and work the best in settings where the target task has limited data.}}
\label{tbl:analysis}
\normalsize
\end{table*}

We first choose the \texttt{door open} task with three different combinations of dataset size and data quality of the task-specific data with reward labels:
\begin{itemize}
    \item 2k transitions with the expert-level performance (i.e. \textbf{high-quality data with limited sample size and narrow coverage})
    \item 2k transitions with medium-level performance (i.e. \textbf{medium-quality data with limited sample size and narrow coverage})
    \item a medium-replay dataset with 152k transitions (i.e. \textbf{medium-quality data with sufficient sample size and broad coverage}).
\end{itemize}
 We share the same data from the other three tasks, \texttt{door close}, \texttt{drawer open} and \texttt{drawer close} as in Table~\ref{tbl:gym}, which are . As shown in Table~\ref{tbl:analysis}, both UDS and CDS+UDS are able to outperform No Sharing in the three settings, suggesting that increasing the coverage of the offline data as suggested by our theory does lead to performance boost in wherever we have limited good-quality data (expert), limited medium-quality data (medium) and abundant medium-quality data (medium-replay). It’s worth noting that UDS and CDS+UDS significantly outperform No Sharing in the limited expert and medium data setting whereas in the medium-replay setting with broader coverage, CDS+UDS outperforms No sharing but UDS fails to achieve non-zero success rate. Such results suggest that UDS and CDS+UDS can yield greater benefit when the target task doesn’t have sufficient data and the number of relabeled data is large. The fact that UDS is unable to learn on medium-replay datasets also suggests that data sharing without rewards is less useful in settings where the coverage of the labeled offline data is already quite broad.}

\subsection{\rebuttal{D4RL Hopper Diagnostic Study on Varying Unlabeled Dataset Size}}
\label{app:unlabeled_dataset_size_analysis}

\begin{table*}[t!]
\centering
\scriptsize
\resizebox{\textwidth}{!}{\begin{tabular}{l|l|l|rr|r}
\toprule
 \textbf{Environment} & \textbf{Labeled dataset type / size} & \textbf{Unlabeled dataset type / size}   & \textbf{CDS+UDS}   & \textbf{UDS}          & \textbf{Sharing All (oracle)} \\
 \midrule
 & random / 10k transitions & expert / 10k transitions  & \textbf{10.1} & 10.0 & 71.9\\
 D4RL hopper & random / 10k transitions & expert / 100k transitions  & \textbf{105.8} & 81.8 & 96.3\\
& random / 10k transitions & expert / 1M transitions  & \textbf{102.3} & 97.0 & 102.8\\

\bottomrule
\end{tabular}}
\caption{\footnotesize Ablation study on the unlabeled dataset size ranging from 10k to 1M transitions in the single-task hopper domain. We bold the best method without true reward relabeling.}
\label{tbl:hopper_ablation}
\normalsize
\end{table*}

Following the discussion in Section~\ref{sec:empirical_analysis}, we further study the setting where relabeled data has higher quality than the labeled data in the single-task hopper task by varying the amount of unlabeled data within the range of (10k, 100k, 1M) transitions. We pick the case where labeled data is random and unlabeled data is expert for such ablation study. As shown in Table~\ref{tbl:hopper_ablation}, as the unlabeled (expert) dataset size decreases, the result of UDS drops significantly whereas Sharing All retains a reasonable level of performance. This ablation suggests that as the effective dataset size decreases, the benefit of reducing sampling error is reduced and no longer able to outweigh the reward bias as indicated in our theoretical analysis. Moreover, Table~\ref{tbl:hopper_ablation} also suggests that, in the setting with medium unlabeled dataset size (100k transitions), CDS+UDS is able to prevent the performance drop as seen in UDS and even outperforms Sharing All. However, CDS+UDS cannot successfully tackle the case where there are only 10k unlabeled transitions, suggesting that the reward bias optimized by CDS+UDS is still detrimental in the limited unlabeled data regime.

\subsection{\final{D4RL Hopper Ablation Study on Reward Learning Methods with Varying Labeled Dataset Size and Quality}}
\label{app:reward_learning_ablation}

\final{We performed an ablation that varies the labeled data size (10k \& 20k transitions) and quality (expert / random) for reward learning methods on D4RL hopper, with unlabeled data being 1M \texttt{hopper-medium} transitions.}
\begin{table*}[t!]
\centering
\footnotesize
\begin{tabular}{l|l|l|rrr}
\toprule
\vspace{-0.08cm}
 \textbf{Env} & \textbf{Labeled dataset type / size} & \textbf{Unlabeled dataset type / size}  & \textbf{CDS+UDS}   & \textbf{UDS}          & \textbf{Reward Pred.} \\
 \midrule
 & expert / 10k transitions & medium / 1M transitions & \textbf{78.3}$\pm$ 5.4 & 64.4$\pm$11.7 & 51.7 $\pm$8.4\\
  & expert / 20k transitions & medium / 1M transitions & 95.5$\pm$5.4 & \textbf{99.2}$\pm$3.1 & 96.7$\pm$3.6\\
 hopper & random / 10k transitions & medium / 1M transitions & \textbf{65.8}$\pm$11.3  & 51.9$\pm$2.4 & 33.4$\pm$5.4\\
 \vspace{-0.08cm}
& random / 20k transitions &medium / 1M transitions  & \textbf{69.1}$\pm$4.8 & 59.9$\pm$5.6 & 47.5$\pm$5.9\\
\bottomrule
\end{tabular}
\caption{\footnotesize Ablation study on comparisons between CDS+UDS/UDS and Reward Predictor with varying labeled dataset size and quality in the single-task hopper domain. We bold the best method.}
\label{tab:reward}
\normalsize
\vspace{-0.54cm}
\end{table*}
\final{Observe on the right, as expected, reward learning performs better with more labeled data. Furthermore, while reward learning works well when labeled data is high quality, it fails to perform well when labeled data is of low quality, potentially due to the bias in reward prediction. UDS and CDS+UDS are less sensitive to labeled data quality/size.}

\subsection{\rebuttal{Takeaways from the empirical analysis}}

Given our empirical analysis in Table~\ref{tbl:single_task_analysis}, Table~\ref{tbl:analysis}, Table~\ref{tbl:hopper_ablation} and Table~\ref{tab:reward}, we summarize the applicability of UDS / CDS+UDS under different scenarios for practitioners in Figure~\ref{fig:treeplot} below.

\begin{figure}[ht]
    \centering
    \includegraphics[width=0.95\textwidth]{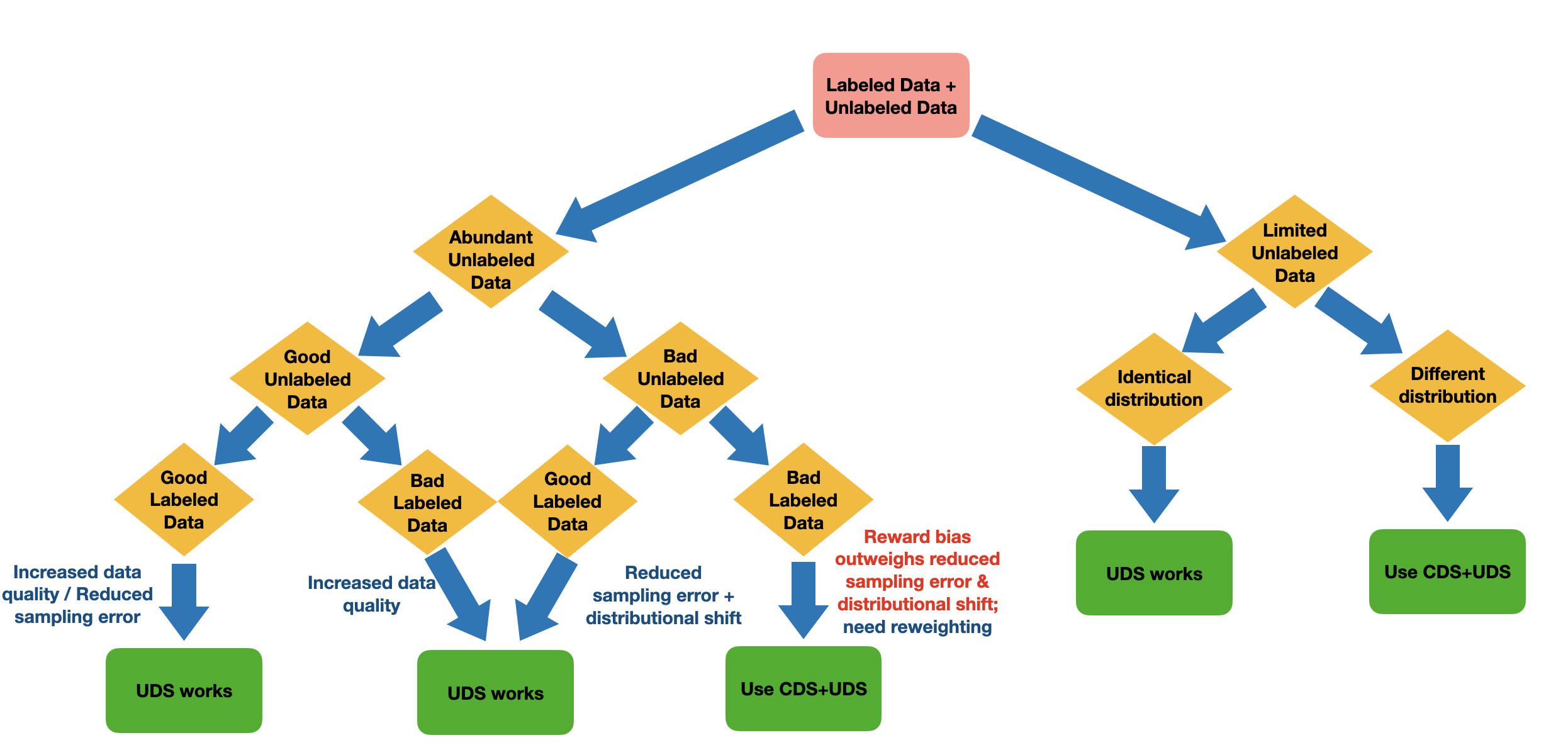}
    \vspace{-0.33cm}
    \caption{\footnotesize  A tree plot that illustrates under which conditions a practitioner should run UDS over apply the optimized reweighting scheme on top of UDS.}
    \vspace{-0.35cm}
    \label{fig:treeplot}
\end{figure}

\section{Additional details on the quality of data shared from other tasks in the multi-task offline RL setting}
\label{app:data_quality}

\begin{table*}[t!]
\centering
\resizebox{0.7\textwidth}{!}{\begin{tabular}{l|l|r}
\toprule
\textbf{Environment} & \textbf{Tasks} & \textbf{Oracle Success Rate of the Shared data}\\
\midrule
&drawer open & 47.4\%\\
& door close & 99.2\%\\
Meta-World& drawer open & 0.1\%\\
& drawer close & 91.6\%\%\\
& \CC \textbf{average} & \CC 59.5\%\\
\midrule
 & medium maze (3 tasks) average  &  4.3\% \\
AntMaze  & large maze (7 tasks) average  & 1.6\% \\
\bottomrule
\end{tabular}}
\caption{\footnotesize Success rate of the data shared from other tasks to the target task determined by the ground-truth multi-task reward function.}
\label{tbl:data_quality}
\normalsize
\end{table*}

We present the success rate of the data shared from other tasks to the target task computed by the oracle multi-task reward function in both the multi-task Meta-World and AntMaze domains in Table~\ref{tbl:data_quality}. Note that the success rate of \texttt{drawer close} and \texttt{door close} are particularly high since for other tasks, the drawer / door is initialized to be closed and therefore the success rate of other task data for these two tasks are almost 100\% as defined by the success condition in the public Meta-World repo. Apart from these two particularly high success rates, the success rates of the shared data are consistently above 0\% across all tasks in both domains. This fact suggests that UDS and CDS+UDS are \emph{not} relabeling with the ground truth reward where the relabeled data are actually all failures but rather performs the conservative bellman backups on relabeled data that is shown to be effective empirically.

\rebuttal{To better understand the performance of UDS under different relabeled data quality, we evaluate the UDS under different success rates of the data relabeled from other tasks in the multi-task Meta-World domain. Specifically, we filter out data shared from other tasks to ensure that the success rates of the relabeled data are 5\%, 50\% and 90\% respectively. We compare the results of UDS on such data compositions to the performance of UDS in Table~\ref{tbl:gym} where the success rate of relabeled data is 59.6\% as shown in Table~\ref{tbl:data_quality}. The full results are in Table~\ref{tbl:data_quality_results}. UDS on relabeled data with 50\% and 90\% success rates achieves similar results compared to original UDS whereas UDS on relabel data with 5\% success rate is significantly worse. Hence, UDS can obtain good results in settings where the relabeled data is of high quality despite incurring high reward bias, but is not helpful in settings where the shared data is of low quality and does not offer much information about solving the target task.}

\begin{table*}[t!]
\centering
\resizebox{\textwidth}{!}{\begin{tabular}{l|l|r|r|r|r}
\toprule
\textbf{Environment} & \textbf{Tasks} & \textbf{UDS} & \textbf{UDS}-5\% relabel success & \textbf{UDS}-50\% relabel success & \textbf{UDS}-90\% relabel success\\
\midrule
&drawer open & 51.9\%$\pm$25.3 & 0.0\%$\pm$0.0\% & 57.3\%$\pm$18.9\%  & 73.3\%$\pm$8.6\%\\
& door close & 12.3\%$\pm$27.6\% & 0.0\%$\pm$0.0\% & 0.0\%$\pm$0.0\% & 0.0\%$\pm$0.0\%\\
Meta-World& drawer open & 61.8\%$\pm$16.3\% & 19.4\%$\pm$27.3\% & 61.0\%$\pm$12.7\% & 56.3\%$\pm$20.3\%\\
& drawer close & 99.6\%$\pm$0.7\% & 66.0\%$\pm$46.7\% & 99.7\%$\pm$0.5\% & 100.0\%$\pm$0.0\%\\
& \CC \textbf{average} & \CC 56.4\%$\pm$12.8\% & \CC 21.4\% $\pm$16.1\% & \CC 54.3\% $\pm$2.0\% & \CC 57.4\%$\pm$3.3\%\\
\bottomrule
\end{tabular}}
\caption{\footnotesize Performance of UDS under different actual success rates of the relabeled data.}
\label{tbl:data_quality_results}
\normalsize
\end{table*}

\section{\rebuttal{Empirical Results of UDS and CDS+UDS in multi-task locomotion domain with dense rewards}}
\label{app:dense_reward}

\rebuttal{In this section, we evaluate UDS and CDS+UDS in the multi-task locomotion setting with dense rewards. We pick the multi-task walker environment as used in prior work~\citep{yu2021conservative}, which consists of three tasks, \texttt{run forward}, \texttt{run backward} and \texttt{jump}. The reward functions of the three tasks are $r(s, a) = v_x - 0.001*\|a\|_2^2$, $r(s, a) = -v_x - 0.001*\|a\|_2^2$ and $r(s, a) = - \|v_x\| - 0.001*\|a\|_2^2 + 10*(z - \text{init z})$ respectively where $v_x$ denotes the velocity along the x-axis and $z$ denotes the z-position of the half-cheetah and $\text{init z}$ denotes the initial z-position. In UDS and CDS+UDS, we relabel the rewards routed from other tasks with the minimum reward value in the offline dataset of the target task. As shown in Table~\ref{tbl:walker}, CDS+UDS and UDS outperform No Sharing and Reward Predictor by a large margin while also performing comparably to CDS and Sharing All. Therefore, CDS+UDS and UDS are able to solve multi-task locomotion tasks with dense rewards.}

\begin{table*}[t!]
\small{
\centering
\vspace*{0.1cm}
\resizebox{\textwidth}{!}{\begin{tabular}{ll|rrrr|rr}
\toprule
\textbf{Environment} &\textbf{Tasks / Dataset type} & \textbf{CDS+UDS}& \textbf{UDS} & \textbf{No Sharing} & \textbf{Reward Predictor} & \textbf{CDS (oracle)} & \textbf{Sharing All (oracle)}\\ \midrule
& run forward / medium-replay & \textbf{880.1}$\pm$108.8 & 665.0$\pm$84.9  & 590.1$\pm$48.6 & 520.7$\pm$373.6 & 1057.9$\pm$121.6 & 701.4$\pm$47.0\\
Multi-Task Walker & run backward / medium & \textbf{717.8}$\pm$78.3 & 689.3$\pm$16.3 & 614.7$\pm$87.3 & 417.3$\pm$235.3 & 564.8$\pm$47.7 & 756.7$\pm$76.7\\
& jump / expert & 1487.7$\pm$177.6 & 1036.0$\pm$247.1  & \textbf{1575.2}$\pm$70.9 & 583.0$\pm$432.0 & 1418.2$\pm$138.4 & 885.1$\pm$152.9\\
& \CC \textbf{average} & \CC \textbf{1028.6}$\pm$76.8 &  \CC 796.7$\pm$106.3  & \CC 926.6$\pm$37.7 & \CC 506.7 $\pm$ 343.6 & \CC 1013.6$\pm$71.5 & \CC 781.0$\pm$100.8\\
\bottomrule
\end{tabular}}
\vspace{-0.3cm}
\caption{\footnotesize Results for multi-task walker experiment with dense rewards. We only bold the best-performing method that does not have access to the true reward during relabeling. CDS+UDS and UDS are able to outperform No Sharing and Reward Predictor while attaining competitive results compared to CDS and Sharing All with oracle rewards.
}
\label{tbl:walker}
}
\end{table*}

\section{\rebuttal{Comparisons of CDS+UDS and UDS to Multi-Task Model-Based Offline RL Approaches}}
\label{app:mbrl}

\rebuttal{In this section, we compare CDS+UDS and UDS to a recent, state-of-the-art model-based offline RL method COMBO~\citep{yu2021combo} in the Meta-World domain. We adapt COMBO to the multi-task offline setting by learning the dynamics model on data of all tasks combined and and performing vanilla multi-task offline training without data sharing using the model learned with all of the data. As shown in Table~\ref{tbl:combo}, CDS+UDS and UDS indeed outperform COMBO in the average task success rate. The intuition behind this is that COMBO is unable to learn an accurate dynamics model for tasks with limited data as in our Meta-World setting.}

\begin{table*}[t!]
\centering
\resizebox{0.8\textwidth}{!}{\begin{tabular}{l|l|r|r|r}
\toprule
\textbf{Environment} & \textbf{Tasks} & \textbf{CDS+UDS} & \textbf{UDS} & COMBO~\citep{yu2021combo}\\
\midrule
& door open & \textbf{61.3\%}$\pm$7.9\% & 51.9\%$\pm$25.3\% & 0.0\%$\pm$0.0\%\\
& door close & \textbf{54.0\%} $\pm$42.5\% & 12.3\%$\pm$27.6\% & 1.1\%$\pm$1.6\%\\
Meta-World& drawer open & \textbf{73.5\%}$\pm$9.6\% & 61.8\%$\pm$16.3\% & 15.7\%$\pm$15.2\%\\
& drawer close & 99.3\%$\pm$0.7\% & \textbf{99.6\%}$\pm$0.7\% & 85.7\%$\pm$13.3\%\\
& \CC \textbf{average} & \CC \textbf{71.2\%} $\pm$ 11.3\% & \CC 56.4\%$\pm$12.8\% & \CC 25.6\%$\pm$6.2\%\\
\bottomrule
\end{tabular}}
\caption{\footnotesize On the multi-task Meta-World domain, we compare CDS+UDS and UDS to the model-based offline RL method COMBO~\citep{yu2021combo} that trains a dynamics model on all of the data and performs model-based offline training using the learned model. CDS+UDS and UDS are able to outperform COMBO by a large margin.}
\label{tbl:combo}
\normalsize
\end{table*}

\section{Comparisons to Pre-Trained Representation Learning on Unlabeled Data}
\label{app:pretrained_reps}
A popular strategy to utilize large amounts of unlabeled data along with some limited amount of labeled data in supervised learning is to utilize the unlabeled data for learning representations, and then running supervised training on the labeled data. A similar strategy has been utilized for RL and imitation learning in some prior work~\citep{yang2021representation,yang2021trail}. On the other hand, the UDS and CDS+UDS strategies take a different route of handling unlabeled data, and it is natural to wonder how these representation learning approaches compare to our approach, UDS, that runs offline RL with the lowest possible reward, with an optional reweighting scheme.

To investigate this, we run experiments comparing our UDS and CDS+UDS approaches to state-of-the-art representation learning approaches for utilizing unlabeled data. First, we compare UDS and CDS+UDS to the ACL~\citep{yang2021representation} approach on the multi-task Meta-World domain. ACL pretrains a representation using a contrastive loss on both the labeled and unlabeled offline datasets and then runs standard multi-task offline RL using the pretrained representation with \textbf{No Sharing}. To handle the multi-task Meta-World domain, we use the version of ACL without inputting reward labels. ACL can be viewed as an alternative to our unlabeled sharing data scheme, which leverages unlabeled data for representation learning rather than sharing it directly. We show the comparison to ACL in Table~\ref{tbl:acl}. UDS and CDS+UDS outperform ACL in the average task performance while ACL is only proficient on drawer-open and drawer-close, and it cannot solve door-open or door-close. This indicates that sharing the unlabeled data, even when labeled with the minimum possible reward is important for offline RL performance, while pretraining representations on the whole multi-task offline dataset might have limited benefit. Also note that, in principle, UDS / CDS+UDS are complementary to ACL and these approaches can be combined together to further improve performance, which we leave as future work.

\begin{table*}[t]
\centering
\begin{tabular}{l|l|rrr}
\toprule
\textbf{Environment} & \textbf{Tasks} & \textbf{CDS + UDS (ours)} & \textbf{UDS (ours)} & \textbf{ACL}\\ \midrule
& door open & \textbf{61.3\%}$\pm$7.9\% & 51.9\%$\pm$25.3\% & 2.8\%$\pm$2.0\%\\
& door close & 54.0\% $\pm$42.5\% & 12.3\%$\pm$27.6\% & 0.0\%$\pm$0.0\%\\
Meta-World& drawer open & 73.5\%$\pm$9.6\% & 61.8\%$\pm$16.3\% & \textbf{83.2\%}$\pm$14.2\%\\
& drawer close & 99.3\%$\pm$0.7\% & 99.6\%$\pm$0.7\% & \textbf{100.0\%}$\pm$0.0\%\\
& \CC \textbf{average} & \CC \textbf{71.2\%} $\pm$ 11.3\% & \CC 56.4\%$\pm$12.8\% & \CC 46.4\%$\pm$3.5\%\\
\bottomrule
\end{tabular}
\vspace{-0.2cm}
\caption{\footnotesize Comparison between \textbf{UDS} / \textbf{CDS+UDS} and the \textbf{ACL}~\citep{yang2021representation} that performs representation learning on the unlabeled data instead of data sharing. Both \textbf{UDS} and \textbf{CDS+UDS} outperform \textbf{ACL} by a significant margin in the average task result, suggesting that sharing the unlabeled data is crucial in improving the performance of multi-task offline RL with unlabeled data compared to only using the data for learning the representation.}
\label{tbl:acl}
\vspace{-0.5cm}
\normalsize
\end{table*}

We also compare the results of UDS and CDS + UDS in the single-task AntMaze medium-play and large-play domains (shown in  Table~\ref{tbl:single_task} in the main paper) to imitation learning methods, TRAIL~\citep{yang2021trail} and other baselines from \citet{yang2021trail}, that also utilize unlabeled data. These methods train a representation using the unlabeled dataset, and then run behavioral cloning to imitate the expert trajectories, which are limited in number. For example, the TRAIL method learns a state-conditioned action representation, by fitting an energy-based model to the transition dynamics of the unlabeled dataset, and then performs downstream imitation learning using the labeled expert dataset. In contrast to these prior approaches, UDS and CDS+UDS do not make any assumptions about how expert the labeled dataset is, and so they are not technically comparable to these imitation learning methods directly. Moreover, any specialized representation learning objective can also be combined with UDS and CDS+UDS to further improve performance. Nevertheless, we hope that a direct comparison to representation learning on unlabeled data combined with downstream imitation will provide informative insights about the potential of UDS and CDS+UDS to effectively leverage unlabeled data. Comparing Table~\ref{tbl:single_task} to Figure 3 in \citet{yang2021trail}, we find that CDS+UDS outperforms the best method from \citet{yang2021trail}, TRAIL, on the antmaze-medium-play task and performs a bit worse on the antmaze-large-play task. CDS+UDS also outperforms all the other methods in Figure 3 of \cite{yang2021trail}. Overall, this implies that UDS and CDS+UDS methods can perform comparably to state-of-the-art representation learning approaches with downstream imitation, without needing any specialized representation learning.

\section{Details of UDS and CDS+UDS}
\label{app:details}

In this section, we include the details of training UDS and CDS+UDS in  Appendix~\ref{app:training_details} as well as details on the environment and datasets used in our experiments in Appendix~\ref{app:env_data_details}. Finally, we discuss the compute information of UDS and CDS+UDS in Appendix~\ref{app:compute_details}.

\subsection{Details on the training procedure}
\label{app:training_details}
We first present our practical implementation of UDS optimizes the following objectives for the Q-functions and the policy in the \emph{single-task offline RL setting}:
\vspace*{-5pt}
\begin{small}
\begin{align*}
    \hat{Q}^{k+1} \leftarrow& \arg\min_{\hat{Q}} \beta\left(\mathbb{E}_{\bs \sim \mathcal{D}_\text{L} \cup \mathcal{D}_\text{U}, \ba \sim \mu(\cdot|\bs)}\left[\hat{Q}(\bs,\ba)\right]- \mathbb{E}_{\bs, \ba \sim \mathcal{D}_\text{L} \cup \mathcal{D}_\text{U}}\left[\hat{Q}(\bs,\ba)\right]\right)\\
    & + \frac{1}{2}\mathbb{E}_{(\bs, \ba, \bs') \sim \mathcal{D}_\text{L} \cup \mathcal{D}_\text{U}}\left[ \left(\hat{Q}(\bs, \ba) - \left(r(\bs, \ba) + \gamma Q(\bs', \ba')\right)\right)^2 \right],
\end{align*}
\end{small}
and
\[
\policy \leftarrow \arg \max_{\policy'} \ \mathbb{E}_{\bs \sim \mathcal{D}_\text{L} \cup \mathcal{D}_\text{U}, \ba \sim \policy'(\cdot | \bs)} \left[\hat{Q}^\policy (\bs, \ba) \right],
\]

Moreover, CDS+UDS optimizes the following objectives for training the critic and the policy with a soft weight:
\vspace*{-5pt}
\begin{small}
\begin{align*}
    \hat{Q}^{k+1} \leftarrow& \arg\min_{\hat{Q}} \beta\left(\mathbb{E}_{\bs \sim \mathcal{D}_\text{L} \cup \mathcal{D}_\text{U}, \ba \sim \mu(\cdot|\bs)}\left[w_{\mathrm{CDS}}(\bs, \ba; \text{U} \rightarrow \text{L})\hat{Q}(\bs,\ba)\right]- \mathbb{E}_{\bs, \ba \sim \mathcal{D}_\text{L} \cup \mathcal{D}_\text{U}}\left[w_{\mathrm{CDS}}(\bs, \ba; \text{U} \rightarrow \text{L})\hat{Q}(\bs,\ba)\right]\right)\\
    & + \frac{1}{2}\mathbb{E}_{(\bs, \ba, \bs') \sim \mathcal{D}_\text{L} \cup \mathcal{D}_\text{U}}\left[ w_{\mathrm{CDS}}(\bs, \ba; \text{U} \rightarrow \text{L})\left(\hat{Q}(\bs, \ba) - \left(r(\bs, \ba) + \gamma Q(\bs', \ba')\right)\right)^2 \right],
\end{align*}
\end{small}
and
\[
\policy \leftarrow \arg \max_{\policy'} \ \mathbb{E}_{\bs \sim \mathcal{D}_\text{L} \cup \mathcal{D}_\text{U}, \ba \sim \policy'(\cdot | \bs)} \left[w_{\mathrm{CDS}}(\bs, \ba; \text{U} \rightarrow \text{L})\hat{Q}^\policy (\bs, \ba) \right],
\]
where $\beta$ is the coefficient of the CQL penalty on distribution shift, $\mu$ is an action sampling distribution that covers the action bound as in CQL. On the hopper domain, when the unlabeled data is random, we use the version of CQL that does not maximize the term $\mathbb{E}_{\bs, \ba \sim \mathcal{D}_\text{L} \cup \mathcal{D}_\text{U}}\left[\hat{Q}(\bs,\ba)\right]$ to prevent overestimating Q-values on low-quality random data and use $\beta = 1.0$. We use $\beta = 5.0$ in the other settings in the hopper domain. On the AntMaze domain, following prior works~\citep{kumar2020conservative,yu2021conservative}, we use the Lagrange version of CQL, where the coefficient $\beta$ is
automatically tuned against a pre-specified constraint value on the CQL loss equal to $\tau = 10.0$. We adopt other hyperparameters used in \citep{yu2021conservative}. We adapt the CDS weight to the single-task setting as follows: $w_{\mathrm{CDS}}(\bs, \ba; \text{U} \rightarrow \text{L}) := \sigma \left(\frac{\Delta(\bs, \ba; \text{U} \rightarrow \text{L})}{\tau} \right)$ where $\Delta(\bs, \ba; \text{U} \rightarrow \text{L}) = \hat{Q}^\pi(\bs, \ba) - P_{k\%}\!\left\{\!\hat{Q}^\pi(\bs', \ba')\!\!: \bs', \ba' \sim \mathcal{D}_\text{L}\!\right\}$ for $(\bs, \ba) \sim \mathcal{D}_\text{U}$. We use $k = 50$ in all single-task domains.

Similarly in the \emph{multi-task offline RL} setting, our practical implementation of UDS optimizes the following objectives for the Q-functions and the policy:
\vspace*{-5pt}
\begin{small}
\begin{align*}
    \hat{Q}^{k+1} \leftarrow& \arg\min_{\hat{Q}} \mathbb{E}_{i\sim[N]}\left[\beta\left(\mathbb{E}_{j \sim[N]}\left[\mathbb{E}_{\bs \sim \mathcal{D}_j, \ba \sim \mu(\cdot|\bs,i)}\left[\hat{Q}(\bs,\ba,i)\right]\right.\right.\right.\\
    &\left.\left.\left.- \mathbb{E}_{\bs, \ba \sim \mathcal{D}_j}\left[\hat{Q}(\bs,\ba, i)\right]\right]\right)\right.\\
    &\left. + \frac{1}{2}\mathbb{E}_{j\sim[N],(\bs, \ba, \bs') \sim \mathcal{D}_j}\left[ \left(\hat{Q}(\bs, \ba, i) - \left(r(\bs, \ba, i)\indicator_{\{j = i\}} + \gamma Q(\bs', \ba')\right)\right)^2 \right]\right],
\end{align*}
\end{small}
\vspace*{-19pt}
and
\[
\policy \leftarrow \arg \max_{\policy'} \ \mathbb{E}_{i \sim [N]}\left[\mathbb{E}_{j\sim[N],\bs \sim \mathcal{D}_j, \ba \sim \policy'(\cdot | \bs, i)} \left[\hat{Q}^\policy (\bs, \ba, i) \right]\right],
\]

We also present the objective of CDS+UDS in the multi-task task as follows:
\vspace*{-5pt}
\begin{small}
\begin{align*}
    \hat{Q}^{k+1} \leftarrow& \arg\min_{\hat{Q}} \mathbb{E}_{i\sim[N]}\left[\beta\left(\mathbb{E}_{j \sim[N]}\left[\mathbb{E}_{\bs \sim \mathcal{D}_j, \ba \sim \mu(\cdot|\bs,i)}\left[w_{\mathrm{CDS}}(\bs, \ba; j \rightarrow i)\hat{Q}(\bs,\ba,i)\right]\right.\right.\right.\\
    &\left.\left.\left.- \mathbb{E}_{\bs, \ba \sim \mathcal{D}_j}\left[w_{\mathrm{CDS}}(\bs, \ba; j \rightarrow i)\hat{Q}(\bs,\ba, i)\right]\right]\right)\right.\\
    &\left. + \frac{1}{2}\mathbb{E}_{j\sim[N],(\bs, \ba, \bs') \sim \mathcal{D}_j}\left[w_{\mathrm{CDS}}(\bs, \ba; j \rightarrow i) \left(\hat{Q}(\bs, \ba, i) - \left(r(\bs, \ba, i)\indicator_{\{j = i\}} + \gamma Q(\bs', \ba')\right)\right)^2 \right]\right],
\end{align*}
\end{small}
\vspace*{-19pt}
and
\[
\policy \leftarrow \arg \max_{\policy'} \ \mathbb{E}_{i \sim [N]}\left[\mathbb{E}_{j\sim[N],\bs \sim \mathcal{D}_j, \ba \sim \policy'(\cdot | \bs, i)} \left[ w_{\mathrm{\methodname}}(\bs, \ba; j \rightarrow i)\hat{Q}^\policy (\bs, \ba, i) \right]\right],
\]
where we use $k = 90$ in the multi-task Meta-World domain and $k = 50$ in the other multi-task domains.

To compute the weight $w_{\mathrm{CDS}}$, we pick $\tau$, i.e. the temperature term, using the exponential running average of $\Delta(\bs, \ba; \text{U} \rightarrow \text{L})$ in the single-task setting or $\Delta(\bs, \ba; j \rightarrow i)$ in the multi-task setting with decay $0.995$ following \cite{yu2021conservative}. Following \cite{yu2021conservative} again, we clip the automatically chosen $\tau$ with a minimum and maximum threshold, which we directly use the values from \cite{yu2021conservative}. We use $[1, \infty]$ as the minimum and maximum threshold for all state-based single-task and multi-task domains whereas the vision-based robotic manipulation domain does not require such clipping.

In the single-task experiments, we use a total batch size of $256$ and balance the number of transitions sampled from $\mathcal{D}_\text{L}$ and $\mathcal{D}_\text{U}$ in each batch. In the multi-task experiments, following the training protocol in \cite{yu2021conservative}, for experiments with low-dimensional inputs, we use a stratified batch with $128$ transitions for each task to train the Q-functions and the policy. We also balance the numbers of transitions sampled from the original task and the number of transitions drawn from other task data. Specifically, for each task $i$, we sample $64$ transitions from $\mathcal{D}_i$ and the remaining $64$ transitions from $\cup_{j \neq i} \mathcal{D}_{j \rightarrow i}$. In CDS+UDS, for each task $i \in [N]$, we only apply $w_\mathrm{CDS+UDS}$ to data shared from other tasks on multi-task Meta-World environments and multi-task vision-based robotic manipulation tasks while we also apply the relabeling weight to transitions sampled from the original task dataset $\mathcal{D}_i$ with 50\% probability in the multi-task AntMaze domain.

Regarding the choices of the architectures, for state-based domains, we use 3-layer feedforward neural networks with $256$ hidden units for both the Q-networks and the policy. In the multi-task domains, we condition the policy and the Q-functions on a one-hot task ID, which is appended to the input state. In domains with high-dimensional image inputs, we adopt the multi-headed convolutional neural networks used in ~\cite{kalashnikov2021mt,yu2021conservative}. We use images with dimension $472 \times 472 \times 3$, extra state features $(g_\text{robot\_status}, g_\text{height})$ and the one-hot task vector as the observations similar \cite{kalashnikov2021mt,yu2021conservative}. Following the set-up in \cite{kalashnikov2018scalable,kalashnikov2021mt,yu2021conservative}, we use Cartesian
space control of the end-effector of the robot in 4D space (3D position and azimuth angle) along with two binary actions to
open/close the gripper and terminate the episode respectively to represent the actions. For more details, see \cite{kalashnikov2018scalable,kalashnikov2021mt}.

\subsection{Details on the environment and the datasets}
\label{app:env_data_details}

In this subsection, we include the discussion of the details the environment and datasets used for evaluating UDS and CDS+UDS. Note that all of our single-task datasets and environments are from the standard benchmark D4RL~\citep{fu2020d4rl} while all of our multi-task environment and offline datasets are from prior work~\citep{yu2021conservative}. We will nonetheless discuss the details to make our work self-contained.  We acknowledge that all datasets with low-dimensional inputs are under the MIT License.

\textbf{Single-task hopper domains.} We use the \texttt{hopper} environment and datasets from D4RL~\citep{fu2020d4rl}. We consider the following three datasets, \texttt{hopper-random}, \texttt{hopper-medium} and \texttt{hopper-expert}. We construct the seven combinations of different data compositions using the three datasets as discussed in Table~\ref{tbl:single_task_analysis}. To construct the combination, we take the first 10k transitions from labeled dataset and concatenate these labeled transitions with the entire unlabeled dataset with 1M transitions.

\textbf{Single-task AntMaze domains.} We use the \texttt{AntMaze} environment and datasets from D4RL~\citep{fu2020d4rl} where we consider two datasets, \texttt{antmaze-medium-play} and \texttt{antmaze-large-play}. These two datasets only contain 1M sub-optimal transitions that navigates to random or fixed locations that are different from the target task during evaluation. We use these two datasets as unlabeled datasets. For the labeled dataset, we use 10 expert demonstrations of solving the target task used in prior work~\citep{yang2021trail}.

\textbf{Multi-task Meta-World domains.} We use the \texttt{door open}, \texttt{door close}, \texttt{drawer open} and \texttt{drawer close} environments introduced in \cite{yu2021conservative} from the public Meta-World~\citep{yu2020metaworld} repo\footnote{The Meta-World environment can be found at the open-sourced repo \url{https://github.com/rlworkgroup/metaworld}}. In this multi-task Meta-World environment, a door and a drawer are put on the same scene, which ensures that all four tasks share the same state space. The environment uses binary rewards for each task, which are adapted from the success condition defined in the Meta-World public repo. In this case, the robot gets a reward of 1 if it solves the target task and 0 otherwise. We use a fixed $200$ timesteps for each episode and do not terminate the episode when receiving a reward of $1$ at an intermediate timestep. We use large datasets with wide coverage of the state space and 152K transitions for the \texttt{door open} and \texttt{drawer close} tasks and datasets with limited (2K transitions), but optimal demonstrations for the \texttt{door close} and \texttt{drawer open} tasks.

We direct use the offline datasets constructed in \cite{yu2021conservative}, which are generated by training online SAC policies for each task with the dense reward defined in the Meta-World repo for 500 epochs. The medium-replay datasets use the whole replay buffer of the online SAC agent until 150 epochs while the expert datasets are collected by the final online SAC policy.

\textbf{Multi-task AntMaze domains.} Following \cite{yu2021conservative}, we use the \texttt{antmaze-medium-play} and \texttt{antmaze-large-play} datasets from D4RL~\citep{fu2020d4rl} and partitioning the datasets into multi-task datasets in an undirected way defined in \cite{yu2021conservative}. Specifically, the dataset is randomly splitted into chunks with equal size, and then each chunk is assigned to a randomly chosen task. Therefore, under such a task construction scheme, the task data for each task is of low success rate for the particular task it is assigned to and it is imperative for the multi-task offline RL algorithm to leverage effective data sharing strategy to achieve good performance. In AntMaze, we also use a binary reward, which provides the agent a reward of +1 when the ant reaches a position within a 0.5 radius of the task goal, which is also the reward used default by \citet{fu2020d4rl}. The terminal of an episode is set to be true when a reward of +1 is observed. We terminate the episode upon seeing a reward of $1$ with the maximum possible $1000$ transitions per episode. Following \cite{yu2021conservative}, we modify the datasets introduced by \citet{fu2020d4rl} by equally dividing the large dataset into different parts for different tasks, where each task corresponds to a different goal position.

\textbf{Multi-task walker domains.} We have presented the details of the multi-task walker environment in Appendix~\ref{app:dense_reward}.

\textbf{Multi-task image-based robotic picking and placing domains.} Following \cite{kalashnikov2021mt,yu2021conservative}, we use sparse rewards for each task. That is, reward 1 is assigned to episodes that meet the success conditions and 0 otherwise. The success conditions are defined in \citep{kalashnikov2021mt}. 
We directly use the dataset used in \cite{yu2021conservative}. Such a dataset is collected by first training a policy for each individual task using QT-Opt~\citep{kalashnikov2018scalable} until the success rate reaches 40\% and 80\% for picking tasks and placing tasks respectively and then combine the replay buffers of all tasks as the multi-task offline dataset. Note that the success rate of placing is higher because the robot is already holding the object at the start of the placing tasks, making the placing easier to solve. The dataset consists of a total number of 100K episodes with 25 transitions for each episode. 

\subsection{Computation Complexity}
\label{app:compute_details}

We train UDS and CDS+UDS on a single NVIDIA GeForce RTX 2080 Ti for one day on the state-based domains. For the vision-based robotic picking and placing experiments, it takes 3 days to train it on 16 TPUs.

\end{document}